\documentclass{article}
\usepackage{iclr2027_conference,times}
\usepackage{url}

\usepackage[T1]{fontenc}
\usepackage[utf8]{inputenc}
\usepackage{amsmath,amsthm}
\usepackage{amssymb}
\usepackage{booktabs,tabularx,array,multirow}
\usepackage{graphicx,xcolor}
\usepackage{colortbl}
\usepackage{enumitem}
\usepackage{tikz}
\usetikzlibrary{arrows.meta,positioning}
\usepackage{float,etoolbox}
\usepackage{algorithm,algpseudocode}
\usepackage{listings}
\definecolor{codecomment}{RGB}{78,139,149}
\definecolor{codeaccent}{RGB}{202,48,126}
\lstdefinestyle{paperpython}{
 language=Python, basicstyle=\ttfamily\fontsize{8.5}{10.5}\selectfont,
 keywordstyle=\color{codeaccent}, commentstyle=\color{codecomment},
 stringstyle=\color{codeaccent}, columns=fullflexible, keepspaces=true,
 showstringspaces=false, numbers=none, frame=none, aboveskip=5pt, belowskip=4pt,
 morekeywords={randn_like,sample_time,drop_condition,moment_loss,moa_loss,
 bio_weights,time_gate,bridge_coefficients,noise_strength,correct_prediction},
 breaklines=false, tabsize=4
}

\usepackage{hyperref}
\newtheorem{proposition}{Proposition}

\newcommand{\E}{\mathbb E}
\newcommand{\R}{\mathbb R}
\newcommand{\KL}{D_{\mathrm{KL}}}
\newcommand{\norm}[1]{\left\lVert #1\right\rVert}
\newcommand{\best}[1]{\textbf{#1}}
\newcommand{\second}[1]{\underline{#1}}
\title{PerturbBridge: \\
Direct Phenotypic Supervision \\
for Cellular Perturbation Generation}
\author{Anonymous authors}
\usepackage{flafter}
\usepackage[misc]{ifsym}
\AtBeginEnvironment{equation}{\footnotesize}
\AtBeginEnvironment{equation*}{\footnotesize}
\AtBeginEnvironment{align}{\footnotesize}
\AtBeginEnvironment{align*}{\footnotesize}
\AtBeginEnvironment{gather}{\footnotesize}
\AtBeginEnvironment{gather*}{\footnotesize}
\AtBeginEnvironment{multline}{\footnotesize}
\AtBeginEnvironment{multline*}{\footnotesize}

\author{Yucheng Yang\textsuperscript{1}, Xiaoxiao Sun\textsuperscript{1}, Huangqingbo Sun\textsuperscript{1}, Elaine Sui\textsuperscript{1}
\\ \textbf{Emma Lundberg\textsuperscript{1},  Serena Yeung-Levy\textsuperscript{1}$^{\textrm{\Letter}}$} \\
\textsuperscript{1} Stanford University}

\iclrfinalcopy 

\begin{document}
\maketitle

\begin{abstract}
    Virtual cell models should capture phenotypes induced by specific perturbations. However, image-based models are typically trained and evaluated to match real-image distributions, which alone does not explicitly prioritize perturbation-specific biological responses. Recent approaches add biological signals through reinforcement-learning post-training or inference-time steering, raising a central question: \textbf{where should biological response information enter the generative pipeline?} We introduce PerturbBridge, a stochastic data-to-data bridge model that directly supervises perturbation-specific phenotypes in predicted perturbed images during training. PerturbBridge jointly optimizes bridge regression and differentiable biological response losses, with biological gradients passing through a frozen image encoder. We further introduce perturbation-response evaluations based on downstream classification and retrieval. Across three microscopy benchmarks for perturbation profiling and subcellular imaging, PerturbBridge surpasses CellFlux in biological fidelity, image quality, and sampling efficiency. On BBBC021, mechanism-of-action classification improves by 5.66 and 7.5 percentage points with morphological and learned visual features, respectively. On JUMP and AllenCell, retrieval Hit@10 and Hit@1 improve by 1.77 and 7.15 percentage points, respectively. On BBBC021 and JUMP, PerturbBridge achieves lower FID with 80\% fewer model evaluations than CellFlux. Overall, these results show that directly supervising perturbation responses during training improves phenotypic fidelity while enabling efficient cellular image generation. Code is available at \url{https://github.com/yyc1215lll-creator/PerturbBridge}.
\end{abstract}


\begin{figure}[!t]
\centering
\includegraphics[width=\linewidth]{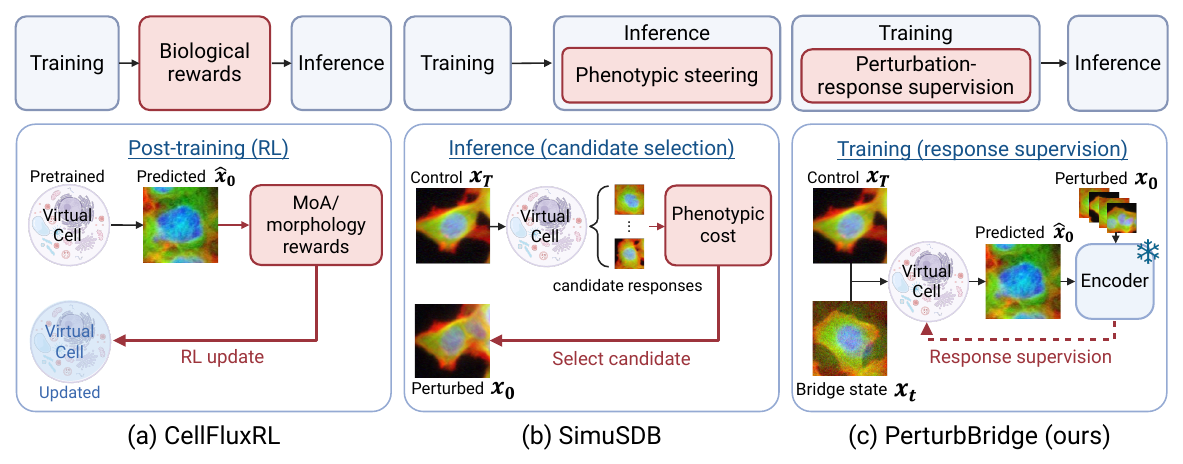}
\caption{\textbf{Three mechanisms for incorporating biological response signals.} (a) CellFluxRL uses biological rewards for RL post-training. (b) SimuSDB uses phenotypic costs to select candidate transitions during sampling. (c) PerturbBridge jointly optimizes bridge regression and biological response losses on predicted perturbed images during generative training. Conditioning inputs and optional test-time selection in CellFluxRL are omitted for clarity.}
\label{fig:where-biology-enters}
\vspace{-4mm}
\end{figure}

\section{Introduction}
\label{sec:intr}

Virtual cell models aim to predict how cells respond to interventions such as drug treatments and gene disruptions~\citep{bunne2024virtual}. For image-based virtual cell modeling, phenotypic profiling characterizes these responses by measuring cell shape, staining intensity, and organelle organization from microscopy images. These measurements support high-throughput compound screening by comparing treated cells with controls. Imaging assays such as Cell Painting~\citep{bray2016cellpainting} capture multiple cellular components, providing complementary views of perturbation-induced changes. Recent methods such as PhenDiff~\citep{bourou2024phendiff}, IMPA~\citep{palma2025impa}, and CellFlux~\citep{zhang2025cellflux} use image-based generative models to translate between cellular conditions or transport control populations to perturbed populations. For biological applications, however, generated images should capture the phenotype associated with the requested perturbation, not merely resemble the overall distribution of real perturbed data. Conditional generative objectives learn associations between images and perturbations from data, while explicit biological response supervision can further emphasize the features and labels relevant to downstream analysis.


Biological signals can be introduced at different stages of the generative pipeline. CellFluxRL (Fig.~\ref{fig:where-biology-enters}(a); \citealp{wu2026cellfluxrl}) uses biological evaluators as rewards for RL post-training, while SimuSDB (Fig.~\ref{fig:where-biology-enters}(b); \citealp{zhang2026simusdb}) uses phenotypic costs to steer candidate transitions during bridge sampling. These approaches motivate a central question: \textbf{where should biological response information enter the generative pipeline?} We investigate direct supervision of perturbation-specific responses during training of the virtual cell model. We introduce \textbf{PerturbBridge} (Fig.~\ref{fig:where-biology-enters}(c)), a stochastic data-to-data bridge model that supervises phenotypic responses directly on predicted perturbed images during generative training.

Specifically, PerturbBridge jointly optimizes bridge regression and differentiable biological response losses on the features of predicted perturbed images. Gradients from the biological response losses pass through a frozen encoder to update the generative model. Targets come from condition-specific statistics of real perturbed training data and available biological labels. Across BBBC021~\citep{caie2010bbbc021}, JUMP~\citep{chandrasekaran2023jump}, and AllenCell~\citep{allen_cell_explorer_drug_perturbation_pilot}, we use condition-moment matching and InfoNCE-style contrastive objectives~\citep{oord2018representation,khosla2020supcon}, adapting the positive and negative sets to the conditions and annotations of each dataset. On BBBC021, the contrastive objective groups positive prototypes by mechanism of action (MoA) to encourage shared phenotypic features across related perturbations.

To test whether this supervision improves perturbation-response fidelity, we evaluate whether generated images preserve the requested biological mechanism and perturbation identity. We use downstream MoA classification and perturbation retrieval, respectively, alongside image-distribution metrics such as FID~\citep{heusel2017fid}. To assess gains beyond the supervising encoder, we use DINOv2~\citep{oquab2024dinov2} representations and CellProfiler~\citep{carpenter2006cellprofiler} features, neither of which is used for biological response supervision.

Across BBBC021~\citep{caie2010bbbc021}, JUMP~\citep{chandrasekaran2023jump}, and AllenCell~\citep{allen_cell_explorer_drug_perturbation_pilot} datasets, PerturbBridge improves biological evaluation scores and image-distribution metrics relative to CellFlux. On BBBC021, DINOv2-based MoA accuracy increases from 67.4\% to 74.9\%, with accuracy also increasing from 72.15\% to 77.81\% under CellProfiler features. Gains also extend to unseen-compound MoA classification under the CellFlux protocol and DINOv2-based perturbation retrieval on JUMP and AllenCell.

Ablations compare PerturbBridge with a Bridge-only model and with biological response supervision applied to perturbed-image predictions reconstructed from CellFlux velocity outputs. These comparisons validate the integrated bridge and supervision design. Using bridge-specific samplers~\citep{zhang2025ecsi,zheng2025dbim}, PerturbBridge achieves lower FID on BBBC021 and JUMP with fewer model evaluations for generation than CellFlux.

We make three contributions:
\begin{enumerate}[leftmargin=*,itemsep=3pt,topsep=3pt,parsep=0pt]
\item \textbf{Direct phenotypic response supervision during generative training.}
We introduce PerturbBridge, which jointly optimizes stochastic-bridge regression and perturbation-specific biological response losses on predicted perturbed images.
\item \textbf{Evaluation of perturbation-specific biological responses.}
We assess generated images through downstream MoA classification and perturbation retrieval using representations not employed for biological response supervision. Our analysis of FID provides theoretical support for evaluating biological responses alongside image-distribution metrics.
\item \textbf{Empirical validation across datasets and sampling efficiency.}
Across BBBC021, JUMP, and AllenCell, PerturbBridge improves biological evaluation scores and image-distribution metrics over CellFlux, with a reduced sampling budget.
\end{enumerate}

\section{Related work and problem setting}
\label{sec:background}
\subsection{Related work on cellular generation and biological response supervision}
\noindent\textbf{Virtual-cell modeling and image generation.}
The broader virtual-cell program spans molecular, transcriptional, and morphological responses~\citep{SLEPCHENKO20101,walkerthevirtualcell,you2004toward,bereketmodelling,hung2024lumic,lamiable2023revealing,10.1093/bioinformatics/btz158}. CellOT~\citep{bunne2023cellot} learns perturbation-response transport from unpaired cell measurements, while GEARS~\citep{roohani2024gears} predicts transcriptional responses using gene relationships. Their biological goals are related, but their outputs and evaluation protocols differ from multichannel image generation. For cellular microscopy, PhenDiff~\citep{bourou2024phendiff} uses conditional diffusion translation, IMPA~\citep{palma2025impa} separates cellular content from perturbation-associated style, and CellFlux~\citep{zhang2025cellflux} uses control-initialized flow matching. These methods are baselines in our experiments.

\noindent\textbf{Biological response supervision and generative model design.}
CellFluxRL~\citep{wu2026cellfluxrl} and SimuSDB~\citep{zhang2026simusdb} incorporate biological signals through RL post-training and inference-time steering, respectively. 
\citet{su2026stochastic} introduce stochastic training perturbations to alleviate sparse supervision in distribution-to-distribution learning. PerturbBridge complements these approaches with perturbation-specific objectives on predicted perturbed images.

\begin{figure}[t]
\centering
\includegraphics[width=\linewidth]{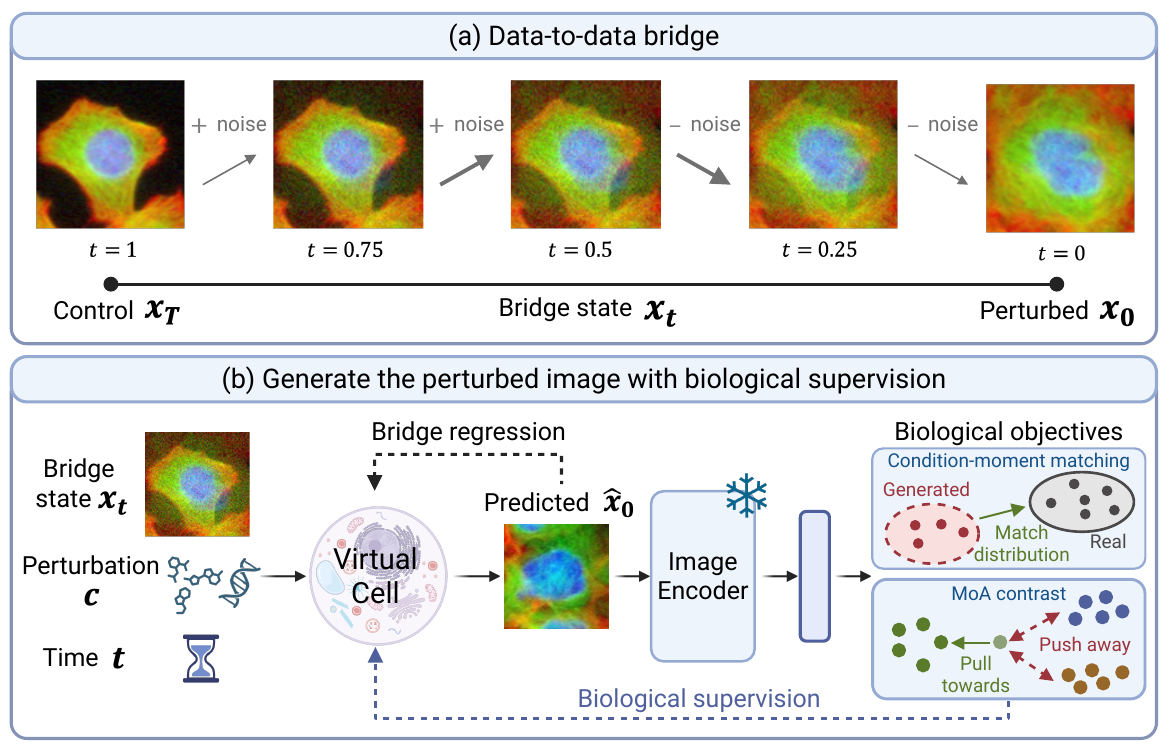}
\caption{\textbf{Training framework of PerturbBridge.}
(a) A stochastic data-to-data bridge connects the control image $x_T$ at $t=1$ to the perturbed image $x_0$ at $t=0$, with noise strongest at the midpoint and vanishing at both ends.
(b) Conditioned on the bridge state $x_t$, perturbation $c$, and time $t$, the generative model predicts the perturbed image $\hat{x}_0$. A frozen image encoder extracts features for MoA contrast and condition-moment matching. These objectives encourage separation across MoA classes and align condition-specific feature means and variances with real perturbed data. Biological gradients pass through the encoder to update the generative model jointly with bridge regression.}
\label{fig:framework}
\end{figure}

\subsection{Problem setting and generative modeling}
\label{sec:problem-setting}
Let $B$ denote experimental context, $C$ the requested perturbation, $X_T$ a control-image random variable, and $X_0$ a perturbed-image random variable. Uppercase symbols denote random variables, and lowercase symbols denote their realizations, such as a specific control image $x_T$ or perturbed image $x_0$. For a fixed context $b$ and perturbation $c$, we seek a conditional generative distribution $K_\theta(x_0\mid x_T,c,b)$ such that
\begin{equation}
\int K_\theta(x_0\mid x_T,c,b)\,p_{\rm ctrl}(x_T\mid b)\,dx_T
\approx p_{\rm pert}(x_0\mid c,b).
\label{eq:task}
\end{equation}
$b$ defines the sampling and evaluation strata, while the control image provides relevant information.

Training pairs perturbed and control images using a joint sampling
distribution $\pi_{c,b}(x_0,x_T)$ for each perturbation $c$ and
experimental context $b$. Pairing within the same batch or plate keeps the experimental context consistent. One possible construction samples the two images independently from their respective populations within each stratum.


\noindent\textbf{Stochastic bridge and image prediction.}
Diffusion bridges connect informative source and target distributions through stochastic intermediate states \citep{zhou2024ddbm,liu2023i2sb}. A Gaussian stochastic interpolant takes the form
\begin{equation}
X_t=\alpha_tX_0+\beta_tX_T+\gamma_t Z,
\qquad Z\sim\mathcal N(0,I),
\label{eq:bridge-general}
\end{equation}
where $Z$ is independent of the image pair. We use a time coordinate $t$, with $t=1$ at the control and $t=0$ at the perturbed target. The coefficients satisfy $(\alpha_0,\beta_0,\gamma_0)=(1,0,0)$ and $(\alpha_1,\beta_1,\gamma_1)=(0,1,0)$. Hence, PerturbBridge predicts the perturbed image from an intermediate bridge state, allowing biological response objectives to act directly on that prediction. Section~\ref{sec:method} specifies the path coefficients and training objective.



\section{Method}
\label{sec:method}
Fig.~\ref{fig:framework} summarizes PerturbBridge, which jointly optimizes bridge regression and biological response losses on predicted perturbed images. We instantiate the stochastic bridge in Equation~\eqref{eq:bridge-general} with a linear interpolation between control and perturbed images and added Gaussian noise \citep{zhou2024ddbm,zhang2025ecsi}:
\begin{equation}
\alpha_t=1-t,\qquad\beta_t=t,\qquad
\gamma_t=2\gamma_{\max}\sqrt{t(1-t)}.
\label{eq:bridge}
\end{equation}
The added noise vanishes at both endpoints and is largest at $t=1/2$. Setting $\gamma_{\max}=0$ yields the deterministic linear path $X_t=(1-t)X_0+tX_T$.

\subsection{Predicting \texorpdfstring{$x_0$}{x0} with biological response losses}
\label{sec:x0}
\noindent\textbf{Why predict the perturbed target image?}
JiT~\citep{li2025backtobasics} motivates direct data prediction through the manifold hypothesis: observed data can occupy a low-dimensional structure within a high-dimensional observation space, whereas noise-containing targets require modeling additional off-manifold variation. This motivates predicting the perturbed image in pixel space.


Image prediction also provides a shared output for pixel-wise regression and biological response supervision. Feature-based losses make errors in encoded phenotypic properties contribute directly to training; App.~\ref{sec:taskmetric} gives a local mathematical interpretation.

\noindent\textbf{Perturbed-image prediction and bridge regression.}
A 1,024-dimensional perturbation embedding is added to the time embedding. The generative model predicts $\hat X_0=D_\theta(X_t,t,c)$ using a preconditioned parameterization.
We train this prediction with weighted pixel-wise mean squared error:
\begin{equation}
\mathcal L_{\rm bridge}
=\E\left[\frac{w(t)}{d_x}\norm{\hat X_0-X_0}^2\right].
\label{eq:x0loss}
\end{equation}
Here, $d_x$ is the number of image entries and $w(t)$ is the preconditioning-based regression weight. The expectation is over training image pairs and perturbation conditions, sampled times $t$, and Gaussian noise $Z$. As in Section~\ref{sec:problem-setting}, $\hat X_0$ denotes a random variable and $\hat x_0$ a realization. The control $X_T$ enters through the bridge state and sampling equations, rather than as a separate network input. App.~\ref{app:sampler} specifies the preconditioning, regression weight, and relation to the fully endpoint-conditioned bridge score.

\noindent\textbf{Biological features and supervision targets.}
Let $f_i=\phi(\hat x_{0,i})\in\R^{d_f}$, where $\phi$ is a frozen Inception-v3 encoder~\citep{heusel2017fid} with $d_f=2048$. The encoder parameters remain fixed while gradients from the biological response losses pass through it to update the generative model. The biological response supervision pipeline uses the first three image channels with differentiable preprocessing; the bridge regression uses all image channels. Reference statistics and prototypes are constructed from real perturbed training data. The biological response objectives use the condition statistics and biological annotations available in each dataset.

\noindent\textbf{Condition-moment matching.}
When real perturbed training data are grouped by perturbation condition, their feature moments provide supervision targets. For each valid condition group $c$ in a minibatch, let $\hat\mu_c,\hat v_c$ be the gate-weighted feature mean and diagonal variance of predicted perturbed images, and $\mu_c,v_c$ the corresponding statistics of real perturbed training data. We use
\begin{equation}
\mathcal L_{\rm mom}=\frac{1}{|\mathcal C_B|d_f}
\sum_{c\in\mathcal C_B}\left[\norm{\hat\mu_c-\mu_c}^2
+\lambda_{\mathrm{var}}\norm{\hat v_c-v_c}^2\right].
\label{eq:mom}
\end{equation}
Here, $\mathcal C_B$ is the set of valid condition groups in the minibatch, and $\lambda_{\mathrm{var}}$ balances mean and variance matching. We apply condition-moment matching on all three datasets; on BBBC021, $\lambda_{\mathrm{var}}=0.5$.

\noindent\textbf{Biological contrast.}
We compute probabilities $q_{ij}$ by applying a softmax with temperature $\tau$ to cosine similarities between predicted-image features and condition prototypes from real perturbed training data. Let $\mathcal P_i$ contain the positive prototypes for sample $i$, defined by the requested condition and available biological annotations; the remaining prototypes are negatives. Following the multi-positive supervised contrastive reduction~\citep{khosla2020supcon},
\begin{equation}
\ell_i^{\rm con}=-\frac{1}{|\mathcal P_i|}
\sum_{j\in\mathcal P_i}\log q_{ij}.
\label{eq:moaloss}
\end{equation}
The batch loss $\mathcal L_{\rm con}$ is a time-gated weighted average of these per-image losses. The loss has the same form across datasets, while the positive and negative sets depend on their condition definitions and annotations. On BBBC021, we use 26 condition prototypes and $\tau=0.1$, with positives sharing the requested MoA. This encourages shared MoA-associated features, while moment matching retains condition-specific constraints.

\noindent\textbf{Joint training objective.}
Across all three datasets, we combine bridge regression with moment matching and biological contrast:
\begin{equation}
\mathcal L_{\rm total}=\mathcal L_{\rm bridge}
+\lambda_{\rm mom}\mathcal L_{\rm mom}
+\lambda_{\rm con}\mathcal L_{\rm con}.
\label{eq:total}
\end{equation}
Biological response losses use time-dependent gating and a gradual weight schedule during training. 
App.~\ref{app:implementation} details the implementation and weight settings.
App.~\ref{app:proofs} decomposes the prototype loss.

\begin{table}[!t]
\centering\small
\caption{BBBC021 MoA with CellProfiler features and DINOv2-L/14 features. Bridge-only omits biological response losses. All scores are percentages. 
Real perturbed data provide a reference.
The CellProfiler-based classifier is trained on real perturbed data and tested on 5,120 generated images.}
\label{tab:cp40}
\setlength{\tabcolsep}{3pt}
\begin{tabular*}{\linewidth}{@{\extracolsep{\fill}}l*{6}{r}@{}}
\toprule
& \multicolumn{3}{c}{\textbf{CellProfiler features}} & \multicolumn{3}{c}{\textbf{DINOv2-L/14}}\\
\cmidrule(lr){2-4}\cmidrule(l){5-7}
Method & Acc.$\uparrow$ & Macro-F1$\uparrow$ & W.-F1$\uparrow$ & Acc.$\uparrow$ & Macro-F1$\uparrow$ & W.-F1$\uparrow$\\
\midrule
IMPA~\citep{palma2025impa} & 57.9 & 53.9 & 60.7 & 57.8 & 36.2 & 60.6\\
CellFlux~\citep{zhang2025cellflux} & 72.2 & 65.8 & 73.2 & 67.4 & 43.4 & 66.6\\
CellFlux+bio & \second{73.3} & 66.6 & 74.1 & 69.8 & 45.4 & 69.1\\ 
Bridge-only & 73.1 & \second{67.6} & \second{74.9} & \second{70.6} & \second{45.8} & \second{70.1}\\
PerturbBridge & \best{77.8} & \best{73.0} & \best{79.3} & \best{74.9} & \best{46.9} & \best{74.3}\\
\cmidrule(l){2-7}
Real Perturbed Data & \textcolor{blue}{80.1} & \textcolor{blue}{76.3} & \textcolor{blue}{81.8} & \textcolor{blue}{78.6} & \textcolor{blue}{57.6} & \textcolor{blue}{79.7}\\
\bottomrule
\end{tabular*}
\vspace{-3mm}
\end{table}

\noindent\textbf{Bridge sampling.}
We use an ECSI-based bridge sampler \citep{zhang2025ecsi}, followed by deterministic DBIM steps \citep{zheng2025dbim}. Each step combines the predicted perturbed image with the control image and analytic path coefficients. Sampling requires no biological evaluation. App.~\ref{app:sampler} provides the update equations and implementation details.

\section{Experiments}
\label{sec:experiments}
\subsection{Experimental setups}
\label{sec:experimental-setup}
\noindent\textbf{Datasets and backbone.}
We evaluate PerturbBridge on three microscopy benchmarks covering perturbation-induced cellular phenotypes and subcellular organization.
BBBC021~\citep{caie2010bbbc021} contains compound-treated cells with mechanism-of-action (MoA) annotations, enabling evaluation of whether generated images preserve features associated with the corresponding biological mechanism.
JUMP~\citep{chandrasekaran2023jump} provides perturbation-labeled images across experimental plates, supporting evaluation of perturbation identity through cross-plate retrieval.
AllenCell~\citep{allen_cell_explorer_drug_perturbation_pilot} captures cellular structures under compound treatments, allowing assessment of changes in subcellular organization and compound retrieval within each imaging target.
Our standard BBBC021 comparison uses 63,781 perturbed training images and 5,120 matched test images.
We adopt the 56.2M-parameter backbone and data preprocessing of CellFlux~\citep{zhang2025cellflux}, and set the bridge noise scale to $\gamma_{\max}=0.1$ on BBBC021 and JUMP.
PerturbBridge is trained for 150, 140, and 240 epochs on BBBC021, JUMP, and AllenCell, respectively.
We compare against baselines with available implementations and additionally report published FID/KID scores for SimuSDB~\citep{zhang2026simusdb} (code not released).


\noindent\textbf{Metrics.}
We assess biological responses through MoA classification on BBBC021 and perturbation retrieval on JUMP and AllenCell, alongside image-distribution quality. Distributional scores alone do not establish response correctness, as illustrated theoretically for pooled and conditional FID in App.~\ref{sec:fid-analysis}. MoA evaluation uses DINOv2~\citep{oquab2024dinov2} and CellProfiler~\citep{carpenter2006cellprofiler} features, while retrieval uses DINOv2 features; none is used for biological response supervision, which relies on Inception-v3. We report classification accuracy, macro-F1, and weighted-F1; retrieval Hit@$k$, mean reciprocal rank (MRR), and mean average precision (mAP); median rank on JUMP; and mean rank on AllenCell. We report Fr\'echet and Kernel Distances in Inception-v3 (FID/KID)~\citep{heusel2017fid,bińkowski2018demystifying} and DINOv2 spaces (FDD/KDD)~\citep{stein2023exposing}. Inception-based MoA scores and FID/KID follow the CellFlux protocol. See App.~\ref{app:metrics}.


\begin{figure}[t]
\centering
\includegraphics[width=\linewidth]{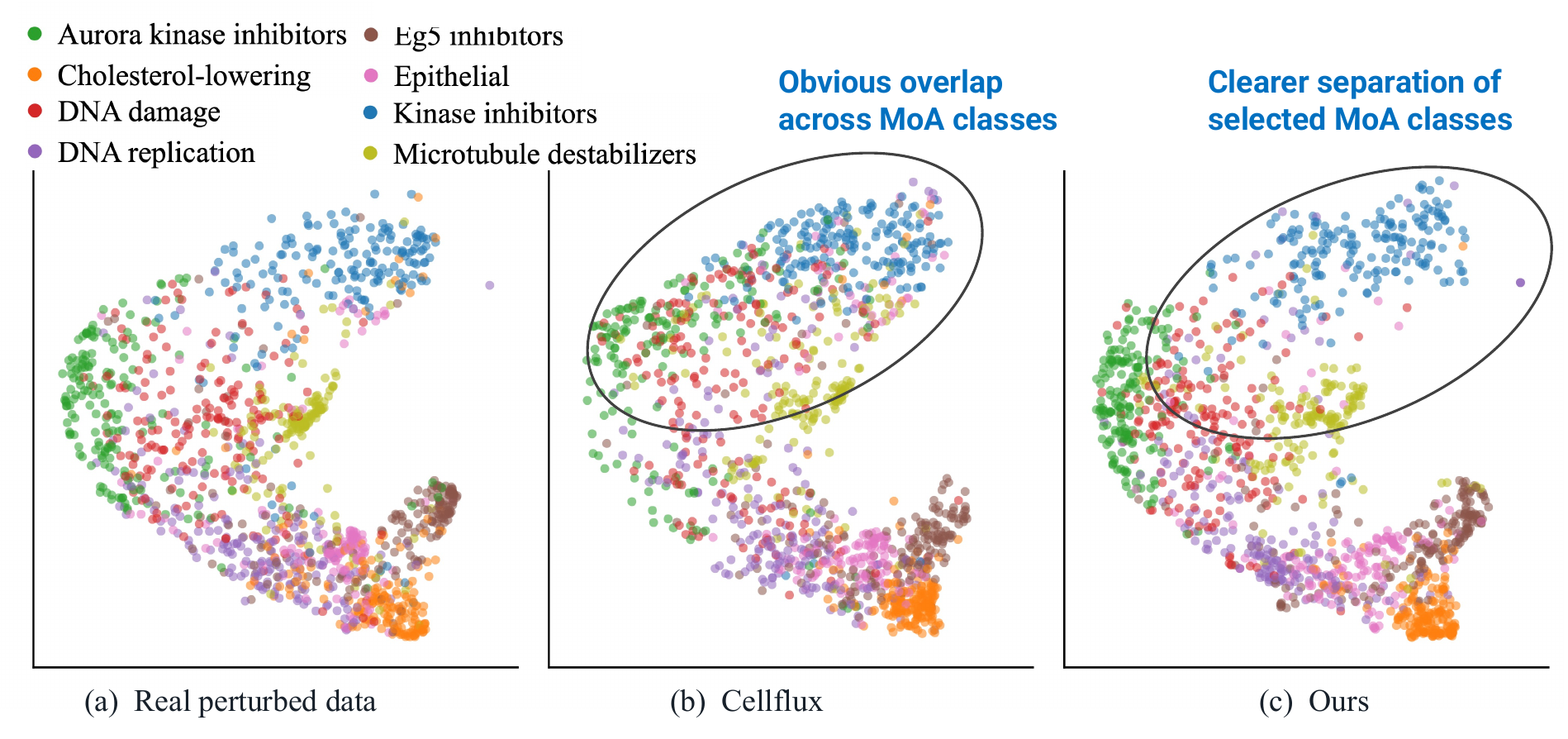}
\vspace{-4mm}
\caption{\textbf{MoA class separation on BBBC021.} (a) Real perturbed data; (b) CellFlux; (c) PerturbBridge (ours). Each point represents a cell image, colored by its MoA label. All panels use a shared UMAP projection of the 512-dimensional hidden features from a frozen MoA classifier. Ellipses highlight overlap across MoA classes in CellFlux and clearer separation in PerturbBridge.}
\label{fig:moa-umap}
\vspace{-4mm}
\end{figure} 

\begin{table}[t]
\centering\scriptsize
\caption{Perturbation retrieval with DINOv2 features. JUMP uses single-cell Mahalanobis retrieval across plates. AllenCell uses condition profiles to retrieve compounds within the same target from profiles of real perturbed training data. Bridge-only omits biological response losses and corresponds to the ECSI-only baseline on AllenCell. Hit rates are percentages; bold and underlined values mark the best and second-best generated-image results within each panel.}
\label{tab:retrieval}
\footnotesize
\setlength{\tabcolsep}{2pt}
\begin{tabular*}{\linewidth}{@{\extracolsep{\fill}}l*{10}{r}@{}}
\toprule
& \multicolumn{6}{c}{\textbf{JUMP}} & \multicolumn{4}{c}{\textbf{AllenCell}}\\
\cmidrule(lr){2-7}\cmidrule(l){8-11}
\midrule
Method
& \shortstack{Hit\\@1$\uparrow$}
& \shortstack{Hit\\@5$\uparrow$}
& \shortstack{Hit\\@10$\uparrow$}
& \shortstack{MRR\\$\times100\uparrow$}
& \shortstack{mAP\\$\times100\uparrow$}
& \shortstack{Median\\rank$\downarrow$}
& \shortstack{Hit\\@1$\uparrow$}
& MRR$\uparrow$
& mAP$\uparrow$
& \shortstack{Mean\\rank$\downarrow$}\\
IMPA~\citep{palma2025impa} & 1.50 & 5.41 & 9.52 & 4.65 & 0.88 & 84 & \multicolumn{1}{|r}{52.14} & 0.73 & 0.73 & 1.64\\
CellFlux~\citep{zhang2025cellflux} & 1.56 & 5.78 & 10.38 & 4.92 & \second{0.92} & 82 & \multicolumn{1}{|r}{\second{62.14}} & 0.78 & 0.78 & 1.58\\
Bridge-only & \second{1.75} & \second{5.98} & \second{10.42} & \second{5.08} & \second{0.92} & \second{79} & \multicolumn{1}{|r}{\second{62.14}} & \second{0.79} & \second{0.79} & \second{1.54}\\
PerturbBridge & \best{2.04} & \best{7.27} & \best{12.15} & \best{5.94} & \best{0.97} & \best{74} & \multicolumn{1}{|r}{\best{69.29}} & \best{0.82} & \best{0.82} & \best{1.48}\\
\midrule
Real Perturbed Data & \textcolor{blue}{2.93} & \textcolor{blue}{7.94} & \textcolor{blue}{12.14} & \textcolor{blue}{6.64} & \textcolor{blue}{1.10} & \textcolor{blue}{77} & \multicolumn{1}{|r}{\textcolor{blue}{76.19}} & \textcolor{blue}{0.87} & \textcolor{blue}{0.87} & \textcolor{blue}{1.29}\\
\bottomrule
\end{tabular*}
\vspace{-4mm}
\end{table}

\subsection{Biological-level evaluation}

\noindent\textbf{PerturbBridge generates images that better preserve MoA-associated phenotypes.}
Table~\ref{tab:cp40} shows improved MoA classification with PerturbBridge on BBBC021 using CellProfiler features and DINOv2-L/14. Relative to CellFlux, accuracy increases from 72.2\% to 77.8\% and from 67.4\% to 74.9\%, corresponding to gains of 5.7 and 7.5 percentage points, respectively. With CellProfiler features, macro-F1 increases from 65.8\% to 73.0\%, and weighted-F1 from 73.2\% to 79.3\%; both DINOv2 F1 scores also improve. These gains hold under both equal class weighting and weighting by class frequency. The two representations capture complementary image properties: CellProfiler measures morphology through features such as shape, intensity, and texture, while DINOv2 provides learned visual features. Neither is used for biological response supervision, which relies on Inception-v3 features. Their consistent improvements therefore support better preservation of MoA-associated phenotypes beyond the supervising representation.

\noindent\textbf{PerturbBridge shows clearer separation of selected MoA classes in feature space.}
The shared UMAP projection in Fig.~\ref{fig:moa-umap} provides qualitative evidence by visualizing the 512-dimensional hidden features of a frozen MoA classifier. Each point represents an image colored by its MoA label, and the projection allows comparison of class organization across real perturbed data, CellFlux, and PerturbBridge. In the highlighted regions, MoA classes form more distinct groups with less overlap in PerturbBridge samples than in CellFlux samples. This suggests better preservation of features that distinguish biological responses, consistent with the improved MoA classification scores.

\begin{figure}[t]
\centering
\includegraphics[width=\linewidth]{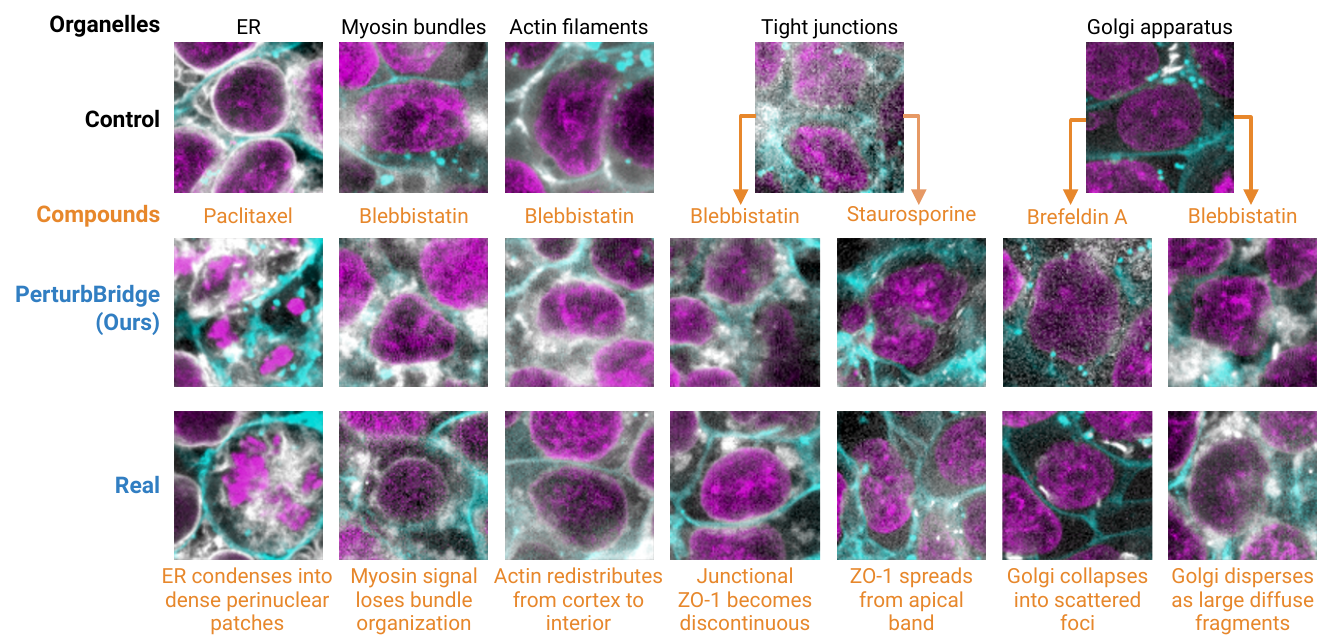}
\vspace{-6mm}
\caption{\textbf{Comparison of perturbation-induced changes on AllenCell.} Columns show seven compound--structure combinations involving the endoplasmic reticulum (ER), myosin bundles, actin filaments, tight junctions, and Golgi apparatus. Control images appear at the top, PerturbBridge samples in the middle (Ours), and real perturbed images at the bottom (Real). Cell nuclei are shown in \textcolor[RGB]{202,116,199}{purple} and plasma membranes in \textcolor[RGB]{165,218,245}{blue}. SNB denotes (S)-nitro-blebbistatin; annotations highlight subcellular changes. Examples are shown in App.~\ref{app:qualitative}.}
\label{fig:generated-real}
\vspace{-6mm}
\end{figure}

\begin{table}[t]
\centering\small
\caption{BBBC021 MoA with Inception-v3 features on the standard and unseen-compound splits. Classification scores are percentages; bold and underlined values mark the best and second-best generated-image results within each split. Dashes indicate unavailable results. Bridge-only omits biological response losses. CellFlux+bio applies matched biological response losses to the predicted perturbed image reconstructed from velocity outputs. Gray rows denote ablation variants.}
\label{tab:moa}
\label{tab:ood}
\setlength{\tabcolsep}{2.5pt}
\begin{tabularx}{\linewidth}{@{}l*{6}{>{\raggedleft\arraybackslash}X}@{}}
\toprule
& \multicolumn{3}{c}{\textbf{Standard split}} & \multicolumn{3}{c}{\textbf{Unseen compounds}}\\
\cmidrule(lr){2-4}\cmidrule(l){5-7}
Method & Acc.$\uparrow$ & Macro-F1$\uparrow$ & W.-F1$\uparrow$ & Acc.$\uparrow$ & Macro-F1$\uparrow$ & W.-F1$\uparrow$\\
\midrule
PhenDiff~\citep{bourou2024phendiff} & 52.6 & 33.6 & 52.1 & 9.6 & 9.3 & 7.4\\
IMPA~\citep{palma2025impa} & 63.7 & 40.2 & 64.8 & 16.0 & 10.0 & 13.1\\
CellFlux~\citep{zhang2025cellflux} & 64.9 & 42.1 & 63.8 & 43.2 & 36.6 & 42.8\\
CellFluxRL~\citep{wu2026cellfluxrl} & 66.2 & \second{42.9} & \second{67.9} & -- & -- & --\\
\rowcolor{black!8}
CellFlux+bio & 66.3 & 42.6 & 65.1 & 44.7 & \second{39.5} & 45.2\\
\rowcolor{black!8}
Bridge-only & \second{67.7} & 42.2 & 67.5 & \second{50.7} & 37.4 & \second{46.1}\\
PerturbBridge & \best{72.4} & \best{46.9} & \best{72.0} & \best{59.0} & \best{47.3} & \best{55.2}\\
\midrule
Real Perturbed Data & \textcolor{blue}{73.6} & \textcolor{blue}{69.7} & \textcolor{blue}{72.1} & \textcolor{blue}{88.0} & \textcolor{blue}{85.0} & \textcolor{blue}{88.0}\\
\bottomrule
\end{tabularx}
\vspace{-3mm}
\end{table}

\noindent\textbf{PerturbBridge generates images with more identifiable perturbation-specific features.}
Table~\ref{tab:retrieval} shows improved DINOv2-based retrieval relative to CellFlux, using single-cell Mahalanobis retrieval across plates on JUMP and condition-profile retrieval within the same target on AllenCell. On JUMP, Hit@10 increases from 10.38\% for CellFlux to 12.15\% for PerturbBridge, while median rank decreases from 82 to 74. On AllenCell, Hit@1 increases from 62.14\% to 69.29\%. Further details of the retrieval protocols are provided in App.~\ref{app:retrieval}.

\noindent\textbf{PerturbBridge captures characteristic perturbation-induced morphological changes.}
Fig.~\ref{fig:generated-real} compares control images, PerturbBridge samples, and real perturbed images on the AllenCell drug perturbation dataset. The selected examples show changes in ER organization, myosin bundles, actin distribution, tight junctions, and Golgi morphology, each corresponding to a phenotype reported for the same compound and cell line. PerturbBridge samples exhibit patterns resembling those in real perturbed images, and the resemblance is compound-specific rather than a generic treated-cell appearance: the Golgi breaks into discrete foci under brefeldin A but into larger, diffuse fragments under SNB, and both patterns are reproduced. This provides qualitative support for the preservation of perturbation-associated phenotypes.


\begin{table}[t]
\centering\footnotesize
\caption{Image-distribution results in Inception-v3 (FID, KID) and DINOv2 (FDD, KDD) feature spaces. Lower is better. KID is multiplied by 100; KDD is unscaled. Bold and underlined values mark the best and second-best results within each dataset and metric. Dashes denote unreported values. CellFluxRL and SimuSDB FID/KID values are from \citet{wu2026cellfluxrl} and \citet{zhang2026simusdb}, respectively. Gray rows denote ablation variants.}
\label{tab:distribution}
\setlength{\tabcolsep}{0.5pt}
\renewcommand{\arraystretch}{1.12}
\begin{tabularx}{\linewidth}{@{}l*{12}{>{\raggedleft\arraybackslash}X}@{}}
\toprule
& \multicolumn{4}{c}{\textbf{BBBC021}} & \multicolumn{4}{c}{\textbf{JUMP}} & \multicolumn{4}{c}{\textbf{AllenCell}}\\
\cmidrule(lr){2-5}\cmidrule(lr){6-9}\cmidrule(lr){10-13}
Method & FID$\downarrow$ & KID$\downarrow$ & FDD$\downarrow$ & KDD$\downarrow$ & FID$\downarrow$ & KID$\downarrow$ & FDD$\downarrow$ & KDD$\downarrow$ & FID$\downarrow$ & KID$\downarrow$ & FDD$\downarrow$ & KDD$\downarrow$\\
\midrule
PhenDiff~\citep{bourou2024phendiff} & 49.5 & 3.10 & 250.1 & 1.55 & 49.3 & 5.10 & 247.3 & 2.76 & 48.8 & 2.75 & 192.6 & 1.56\\
IMPA~\citep{palma2025impa} & 33.7 & 2.60 & 293.9 & 2.57 & 14.6 & 1.10 & 371.1 & 4.43 & 37.6 & 2.25 & 424.1 & 4.98\\
CellFlux~\citep{zhang2025cellflux} & 18.8 & 1.61 & 207.2 & 1.73 & 9.0 & 0.63 & \second{123.1} & \second{1.35} & 28.4 & 1.46 & 181.2 & 1.71\\
CellFluxRL~\citep{wu2026cellfluxrl} & 24.0 & 1.40 & -- & -- & -- & -- & -- & -- & -- & -- & -- & --\\
SimuSDB~\citep{zhang2026simusdb} & \second{16.5} & 1.32 & -- & -- & 8.7 & 0.58 & -- & -- & -- & -- & -- & --\\
\rowcolor{black!8}
CellFlux+bio & 18.4 & 1.59 & 206.0 & 1.75 & 9.8 & 0.70 & 130.5 & 1.46 & \best{25.9} & \second{0.99} & \second{162.0} & \second{1.34}\\
\rowcolor{black!8}
Bridge-only & 17.3 & \second{1.02} & \best{165.7} & \best{1.21} & \best{8.2} & \second{0.47} & 139.1 & 1.52 & 27.0 & 1.27 & 171.0 & 1.58\\
PerturbBridge & \best{16.4} & \best{0.91} & \second{173.8} & \second{1.30} & \second{8.3} & \best{0.38} & \best{108.8} & \best{1.04} & \second{26.7} & \best{0.97} & \best{146.3} & \best{1.08}\\
\bottomrule
\end{tabularx}
\vspace{-3mm}
\end{table}

\begin{figure}[t]
\centering
\begin{minipage}[c]{0.30\linewidth}
\caption{Sampling-budget comparison on BBBC021 and JUMP. Lower FID is better. Gray bars: CellFlux; blue bars: PerturbBridge. Bar labels report FID; horizontal axes indicate model evaluations (NFE).}
\label{fig:sampling-nfe}
\end{minipage}\hfill
\begin{minipage}[c]{0.68\linewidth}
\includegraphics[width=\linewidth]{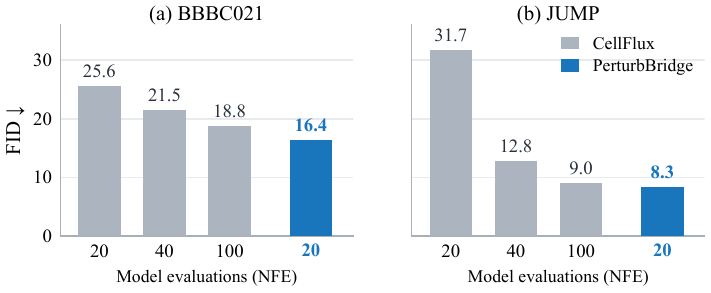}
\end{minipage}
\vspace{-6mm}
\end{figure}

\noindent\textbf{PerturbBridge improves MoA classification on standard and unseen-compound splits.}
Table~\ref{tab:moa} compares MoA classification results on BBBC021 using the released CellFlux classifier with Inception-v3 features. For comparability, we follow the standard and unseen-compound experimental settings of CellFlux \citep{zhang2025cellflux}, including its data splits, preprocessing, and evaluation protocol. Relative to CellFlux, PerturbBridge improves accuracy from 64.9\% to 72.4\% on the standard split and from 43.2\% to 59.0\% on unseen compounds. On the standard split, macro-F1 increases from 42.1\% to 46.9\%, and weighted-F1 from 63.8\% to 72.0\%. On unseen compounds, the corresponding scores increase from 36.6\% to 47.3\% and from 42.8\% to 55.2\%. These improvements show that the MoA classification gains also hold under the CellFlux evaluation protocol and extend to compounds excluded from training.

\noindent\textbf{PerturbBridge improves image-distribution fidelity alongside biological task performance.}
Table~\ref{tab:distribution} compares image-distribution quality across all three datasets in Inception-v3 and DINOv2 feature spaces. Relative to CellFlux, PerturbBridge reduces FID from 18.8 to 16.4 on BBBC021, from 9.0 to 8.3 on JUMP, and from 28.4 to 26.7 on AllenCell. KID also decreases on all three datasets. The improvements extend to DINOv2 features: FDD decreases from 207.2 to 173.8, from 123.1 to 108.8, and from 181.2 to 146.3, respectively, with lower KDD on each dataset as well. On unseen BBBC021 compounds, FID decreases from 42.0 to 29.5. Together, these results support improved distributional agreement with real perturbed data across datasets and feature representations.

\subsection{Ablation studies}
\noindent\textbf{Biological response supervision improves MoA preservation.}
Tables~\ref{tab:moa} and~\ref{tab:distribution} compare PerturbBridge with the Bridge-only variant, which omits biological response supervision. On BBBC021, adding this supervision increases MoA accuracy from 67.7\% to 72.4\% on the standard split and from 50.7\% to 59.0\% on unseen compounds, with improvements in both macro-F1 and weighted-F1. On the standard split, FID decreases from 17.3 to 16.4 and KID from 1.02 to 0.91. These results show that biological response supervision improves MoA preservation beyond bridge regression alone, including for compounds excluded from training.

\noindent\textbf{PerturbBridge improves most evaluated metrics over CellFlux+bio.}
Tables~\ref{tab:moa} and~\ref{tab:distribution} compare PerturbBridge with CellFlux+bio, which applies biological response supervision to perturbed images reconstructed from CellFlux velocity predictions. On BBBC021, CellFlux+bio reaches 66.3\% MoA accuracy, compared with 72.4\% for PerturbBridge. On JUMP, its FID is 9.8, compared with 8.3 for PerturbBridge. Together with the Bridge-only comparison, these results support the combined design of direct perturbed-image prediction and biological response supervision.

\noindent\textbf{PerturbBridge achieves lower FID with fewer model evaluations.}
Fig.~\ref{fig:sampling-nfe} compares CellFlux at 20, 40, and 100 model evaluations (NFE) with PerturbBridge at 20 NFE. At the matched budget of 20 NFE, PerturbBridge achieves FID values of 16.4 on BBBC021 and 8.3 on JUMP, compared with 25.6 and 31.7 for CellFlux. These scores also improve on CellFlux at 100 NFE, which reaches 18.8 and 9.0, respectively. Both methods use the same logical NFE convention; with classifier-free guidance, each logical evaluation uses two network forwards (App.~\ref{app:sampler}).


\section{Conclusion}
We introduce PerturbBridge, a stochastic data-to-data bridge model that directly supervises phenotypic responses in predicted perturbed images during generative training. By jointly optimizing bridge regression and differentiable biological response losses through a frozen encoder, PerturbBridge incorporates biological objectives without requiring RL post-training or biological evaluation during sampling. Across BBBC021, JUMP, and AllenCell, the model improves MoA classification, perturbation retrieval, and image-distribution metrics relative to CellFlux. Gains under CellProfiler and DINOv2 representations extend beyond the features used for supervision. Ablations further demonstrate the contribution of biological response supervision. With a bridge-specific sampler, PerturbBridge achieves lower FID than the 100-NFE CellFlux baseline using 20 model evaluations. These findings support direct biological response supervision as a practical approach to cellular perturbation generation and highlight the value of evaluating generated images through downstream biological tasks alongside distributional metrics.



\bibliographystyle{iclr2027_conference}
\bibliography{references}

\clearpage
\appendix
\onecolumn
\section{Mathematical support for the motivation and method}
\label{app:proofs}
\label{sec:theory}

\subsection{FID ranking reversal: statement and proof}
\label{sec:motivation}
\label{sec:fid-analysis}
Shared appearance and experimental context can dominate image features, while perturbation-specific changes may be subtle \citep{chandrasekaran2024matched,moshkov2024profiling}. Fig.~\ref{fig:fid-biology} illustrates how matching non-target appearance can coexist with an incorrect phenotype. The construction below shows that both pooled and conditional FID can favor biologically incorrect outputs. This concerns what the metrics certify, rather than whether ideal conditional generative learning can recover the target distribution.

\begin{figure}[htbp]
\centering
\includegraphics[width=\linewidth]{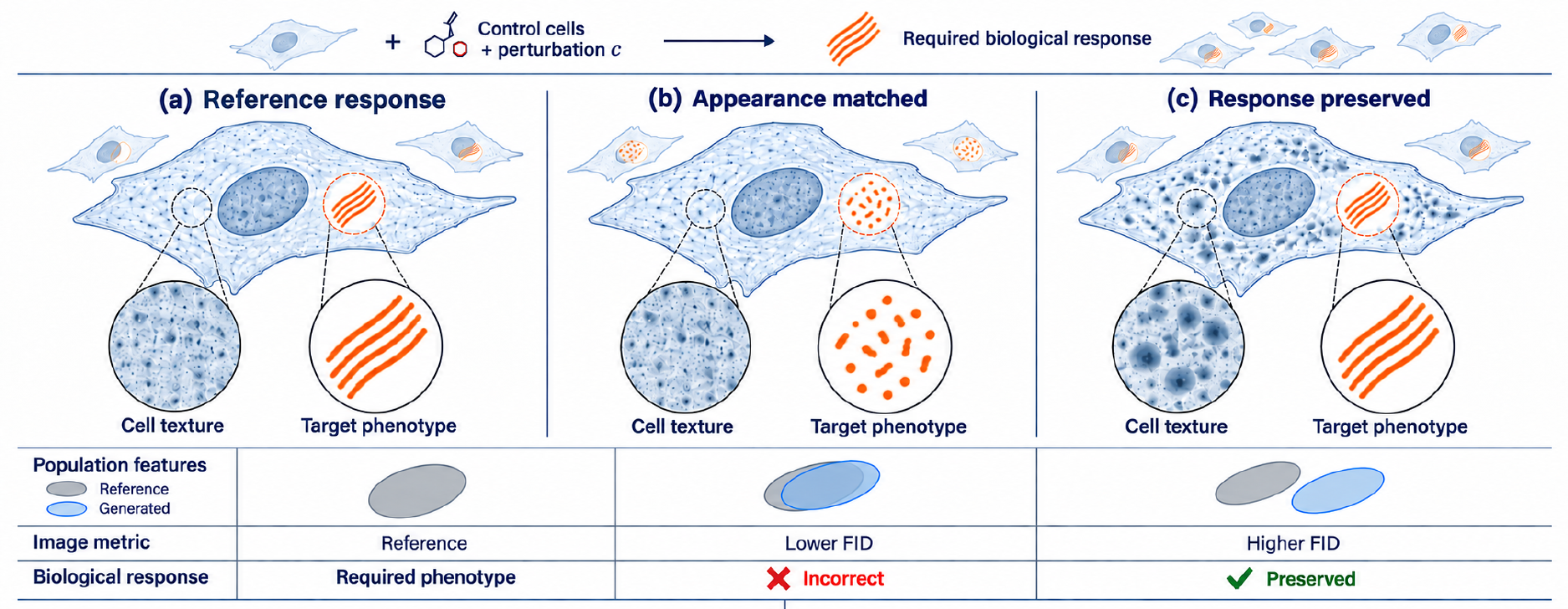}
\caption{\textbf{Image fidelity and biological correctness.} Orange fibers depict the target phenotype; blue-gray texture depicts non-target intracellular appearance. Paired enlargements isolate both features. Relative to (a), (b) matches non-target appearance but disrupts the phenotype, whereas (c) preserves the phenotype despite a texture shift. Population-feature ellipses schematize the resulting disagreement between FID and biological correctness; the FID ordering follows Proposition~\ref{prop:fid-risk}. Cells and ellipses are conceptual illustrations.}
\label{fig:fid-biology}
\end{figure}

Let $P$ denote the real perturbed distribution, with equally likely conditions $C\in\{-1,+1\}$. Let $U\sim\mathcal N(0,I_d)$ be independent of $C$, $\delta>0$, and $a\in\mathbb R^d\setminus\{0\}$. The final feature coordinate determines the biological label through $h(x)=\operatorname{sign}(x_{d+1})$. Compare
\begin{equation}
X=(U,\delta C),\qquad
Q^{\rm keep}:\ \hat X=(U+a,\delta C),\qquad
Q^{\rm flip}:\ \hat X=(U,-\delta C).
\label{eq:toy-data}
\end{equation}
The first model preserves the requested label but shifts unrelated appearance; the second reverses the label while preserving the pooled image distribution.

\begin{proposition}[FID and biological utility can rank models oppositely]
\label{prop:fid-risk}
Using identity features, and defining conditional FID as the average FID within each condition,
\begin{align}
\operatorname{FID}_{\rm id}(P_X,Q_X^{\rm flip})&=0,
&\operatorname{FID}_{\rm id}(P_X,Q_X^{\rm keep})&=\norm a^2,\\
\operatorname{cFID}_{\rm id}(P,Q^{\rm flip})&=4\delta^2,
&\operatorname{cFID}_{\rm id}(P,Q^{\rm keep})&=\norm a^2.
\label{eq:fid-risk-reversal}
\end{align}
If $0<4\delta^2<\norm a^2$, both metrics prefer $Q^{\rm flip}$, although its predicted biological labels are always wrong and those of $Q^{\rm keep}$ are always correct.
\end{proposition}

\noindent\textbf{Proof.}
The pooled means are zero for $P$ and $Q^{\rm flip}$ and $(a,0)$ for $Q^{\rm keep}$, with common covariance $\operatorname{diag}(I_d,\delta^2)$. The covariance contribution to FID therefore vanishes, leaving squared mean differences of zero and $\norm a^2$. Conditional on $C=c$, the means are $(0,c\delta)$, $(a,c\delta)$, and $(0,-c\delta)$, with common covariance $\operatorname{diag}(I_d,0)$, giving conditional FIDs $\norm a^2$ and $4\delta^2$. The signs of the final coordinates establish label correctness. This is a feature-space illustration, not an empirical result on the cellular datasets.

\subsection{Small image errors can carry large biological consequences}
\label{sec:taskmetric}
In Equation~\eqref{eq:toy-data}, removing the phenotype coordinate leaves normalized squared error $\delta^2/(d+1)$ but removes all condition information. A feature-sensitive loss can assign weight to this low-energy coordinate.

For image error $e=\hat x_0-x_0$, a differentiable encoder $\phi$ with Jacobian $J_\phi$ at $x_0$, and $\lambda\geq0$, the local expansion is
\begin{equation}
\norm e^2+\lambda\norm{\phi(x_0+e)-\phi(x_0)}^2
=e^\top\bigl(I+\lambda J_\phi^\top J_\phi\bigr)e+o(\norm e^2).
\label{eq:metric}
\end{equation}
This follows by substituting the first-order expansion of $\phi$ and squaring. The added term weights directions captured by the encoder. This local illustration explains the role of feature-sensitive supervision; the implemented losses use condition moments and prototypes rather than paired feature regression.

\subsection{The implemented loss explicitly constrains label association}
For a predicted-image feature $Z$ and requested MoA label $A$, let $\mathcal P_A$ be the positive prototypes, $K_A=|\mathcal P_A|$, $q_A(Z)=\sum_{j\in\mathcal P_A}q_j(Z)$, and $r_j(Z)=q_j(Z)/q_A(Z)$ within this set.
\begin{proposition}[Biological classification inside prototype contrast]
\label{prop:contrast}
For strictly positive prototype probabilities, the multi-positive loss satisfies
\begin{equation}
\ell(Z,A)=-\frac{1}{K_A}\sum_{j\in\mathcal P_A}\log q_j(Z)
=-\log q_A(Z)+\log K_A+\KL(U_A\Vert r_A(Z)),
\label{eq:decomp}
\end{equation}
where $U_A$ is uniform on the positive set.
\end{proposition}
\noindent\textbf{Proof and interpretation.}
Substitute $q_j=q_A r_j$ into the average negative log probability. The first term encourages probability mass on the requested MoA; the KL term distributes it across positive prototypes, while $\log K_A$ is constant for a given positive set. Condition-moment matching separately constrains compound-specific statistics. This decomposition describes the training objective and does not by itself guarantee classification performance on final generated images.

\section{Bridge training, design, and sampling}
\label{app:sampler}
\subsection{Path construction and the role of prior bridge models}
We use the path and coefficients in Equations~\eqref{eq:bridge-general} and~\eqref{eq:bridge}, with $t=1$ at the control and $t=0$ at the perturbed target. Sampling an image pair, an interior time, and Gaussian noise constructs a training state without simulating a trajectory.

I$^2$SB derives an analytic Gaussian bridge conditional on image pairs \citep{liu2023i2sb}, while DDBM learns a bridge score with input/output preconditioning and weighted regression \citep{zhou2024ddbm}. DBIM reuses a trained predictor for implicit sampling \citep{zheng2025dbim}, and ECSI separates path coefficients, preconditioning, and sampling stochasticity \citep{zhang2025ecsi}. We use the linear ECSI path and apply biological response supervision to predicted perturbed images.

\subsection{Conditional distributions and the bridge score}
\label{app:bridge-score}
At an interior time, the path has Gaussian density $q_t(x\mid x_0,x_T)$ with mean $\alpha_t x_0+\beta_t x_T$ and covariance $\gamma_t^2I$. Integrating over the target distribution induced by the training coupling gives $q_t(x\mid x_T,c)$. Differentiating this mixture expresses its score through the posterior mean of the perturbed target:
\begin{align}
m_t(x,x_T,c)&=\E[X_0\mid X_t=x,X_T=x_T,C=c],\nonumber\\
\nabla_x\log q_t(x\mid x_T,c)
&=-\frac{x-\alpha_t m_t(x,x_T,c)-\beta_t x_T}{\gamma_t^2}.
\label{eq:bg-bridge-score}
\end{align}
This identity assumes positive interior noise and differentiation under the integral. It conditions explicitly on $X_T$, whereas the implemented predictor receives $(X_t,t,c)$. App.~\ref{app:predictor-conditioning} explains this distinction and the condition under which the conditional means coincide.

\subsection{The implemented regression and training configuration}
The preconditioned image prediction is
\begin{equation}
D_\theta(X_t,t,c)=c_{\rm skip}(t)X_t+c_{\rm out}(t)
 F_\theta(c_{\rm in}(t)X_t,c_{\rm noise}(t),c).
\label{eq:preconditioning}
\end{equation}

For our path, $\alpha_t=1-t$, $\beta_t=t$, and $\gamma_t=2\gamma_{\max}\sqrt{t(1-t)}$. The BBBC021 preconditioner uses $S=\sigma_{\rm data}^2=0.25$ and $C_{xy}=0.125$:

\begin{align}
H_t&=(\alpha_t^2+\beta_t^2)S+2\alpha_t\beta_t C_{xy}+\gamma_t^2,
&c_{\rm in}(t)&=H_t^{-1/2},\\
c_{\rm skip}(t)&=\frac{\alpha_t S+\beta_t C_{xy}}{H_t},
&c_{\rm out}(t)&=\sqrt{\frac{\beta_t^2(S^2-C_{xy}^2)+S\gamma_t^2}{H_t}},\\
c_{\rm noise}(t)&=250\log(t+10^{-44}),
&w(t)&=c_{\rm out}(t)^{-2}.
\label{eq:app-precond}
\end{align}

Here $C_{xy}$ is a fixed preconditioning parameter. The raw U-Net output is combined with the skip term to produce $D_\theta$, as in Equation~\eqref{eq:preconditioning}; the regression averages squared error over image entries with weight $w(t)$. The bridge and sampler retain the control image explicitly. The implemented U-Net uses the inputs specified in Section~\ref{sec:x0}; the distinction from a fully endpoint-conditioned regression is stated below.

Table~\ref{tab:bridge-training} summarizes the BBBC021 architecture, optimization, and grouped-sampling settings. Condition groups provide multiple images for estimating the feature moments used by biological response supervision; classifier-free condition dropout also supplies unconditional training batches.

\begin{table}[!htbp]
\centering\small
\caption{BBBC021 training configuration.}
\label{tab:bridge-training}
\begin{tabularx}{\linewidth}{@{}lX@{}}
\toprule
Component & Setting\\\midrule
Backbone & 56,201,219-parameter ADM/U-Net; RGB $96\times96$\\
Architecture & Base width 128; four residual blocks; multipliers $(2,2,2)$; attention resolution 48; one head; dropout 0.3\\
Blocks & Scale--shift normalization; new attention order; convolutional resampling and residual up/downsampling disabled\\
Perturbation & 1,024-dimensional molecular embedding; local-batch condition dropout 0.2\\
Batching & Eight devices, 16 cells each; eight cells per compound; global batch 128\\
Optimizer & RAdam; learning rate $10^{-4}$; weight decay 0; default betas; no learning-rate annealing\\
Precision and EMA & FP16 with gradient scaling; model EMA 0.999\\
Time and path noise & Uniform $t\in[0.0001,0.9999)$; $\gamma_{\max}=0.1$\\
Training budget & 150 epochs, 499 grouped batches per epoch\\
\bottomrule
\end{tabularx}
\end{table}
PerturbBridge is trained for 140 epochs on JUMP and 240 epochs on AllenCell.

The 499 batches per epoch are a grouped-sampling budget, $\lceil63781/128\rceil$. Treated cells are combined with randomly sampled controls from the same experimental batch. Each image receives the original pixel dequantization $(\mathrm{pixel}+U[0,1))/255$, independent horizontal/vertical flips with probability 0.3, and mapping to approximately $[-1,1]$. Target and control images pass through these transforms separately. Additional Gaussian noise augmentation of control images is disabled. App.~\ref{app:implementation} gives the biological response losses and their weight schedule.

\subsection{Sampling dynamics and numerical implementation}
At an interior time, the predicted perturbed image determines the residual coordinate:
\begin{equation}
\hat z_t=\frac{x_t-\alpha_t\hat x_0-\beta_t x_T}{\gamma_t}.
\label{eq:noise}
\end{equation}
The deterministic DBIM update combines this residual with the control and predicted perturbed images:
\begin{equation}
x_r=\alpha_r\hat x_0+\beta_r x_T+\gamma_r\hat z_t.
\label{eq:finish}
\end{equation}

For decreasing time, the ECSI-style stochastic update uses
\begin{align}
b_t&=\dot\alpha_t\hat x_0+\dot\beta_t x_T
+\left(\dot\gamma_t+\frac{\epsilon_t}{\gamma_t}\right)\hat z_t,\nonumber\\
x_{t-h}&=x_t-hb_t+\sqrt{2\epsilon_t h}\,\xi,
\qquad \xi\sim\mathcal N(0,I),\quad h>0.
\label{eq:ecsi-main-update}
\end{align}
Dots denote time derivatives. Following ECSI \citep{zhang2025ecsi}, the noise strength for our linear path is $\epsilon_t=2\eta\gamma_{\max}^2$ in the interior, where $\eta$ controls sampling stochasticity. The score-based derivation uses the posterior mean in Equation~\eqref{eq:bg-bridge-score}; the implementation substitutes the learned prediction as discussed below. The final two steps set $\eta=0$ and use the deterministic DBIM update in Equation~\eqref{eq:finish}, which returns $\hat x_0$ at $r=0$.

\noindent\textbf{Numerical grid and updates.}
The sampler uses $N=20$ evaluations at times
\begin{equation}
t_i=\left[t_{\max}^{1/\rho}+\frac{i}{N-1}
\left(t_{\min}^{1/\rho}-t_{\max}^{1/\rho}\right)\right]^\rho,
\quad i=0,\ldots,N-1,
\end{equation}
followed by $t_N=0$, with $t_{\max}=0.9999$, $t_{\min}=0.001$, and $\rho=0.6$. It initializes the state with $x_T$ at the first finite time. Residual-coordinate divisions are evaluated at interior times, where $\gamma_t>0$. The first 18 transitions use the stochastic update, and the final two use Equation~\eqref{eq:finish}. 

\noindent\textbf{Guidance and counting convention.}
The CFG coefficient is fixed at $s=0.2$ for PerturbBridge and all baselines on BBBC021 and JUMP. Under this convention, the guided denoiser returns $(1+s)D_{\rm cond}-sD_{\rm uncond}$, using two separate network forwards. Thus 20 guided-denoiser evaluations correspond to 40 raw forwards. We count logical evaluations consistently with the CellFlux baseline sampling budget.

\noindent\textbf{Conditioning in the implemented predictor.}
\label{app:predictor-conditioning}
The endpoint-conditioned score above uses $\E[X_0\mid X_t,t,c,X_T]$. Our network takes $(X_t,t,c)$ as input and accesses the control through $X_t$. For pure weighted regression with a positive time-only weight, its population optimum is $\E[X_0\mid X_t,t,c]$. These means coincide when $(X_t,t,c)$ is sufficient for $X_0$ given $X_T$. The bridge identity motivates the sampling update. In implementation, we substitute the biologically supervised image prediction for the conditional mean, while retaining the control image explicitly in the update. This substitution is an approximation; exact equality is not established.

\section{Biological metrics and evaluation protocols}
\label{app:metrics}
The protocols below distinguish MoA membership from exact perturbation identity. Evaluators use fixed references from real perturbed data; generated images are evaluated with their requested labels. The standard BBBC021 comparison uses 63,781 treated training images and 5,120 matched test images. The eight excluded compounds are docetaxel, AZ841, cytochalasin D, simvastatin, cycloheximide (listed as ``cyclohexamide'' in the dataset), latrunculin B, epothilone B, and lactacystin. The unseen-compound experiment follows the experimental setting of CellFlux \citep{zhang2025cellflux}, using the same held-out compounds and evaluation protocol; its results are reported separately in Table~\ref{tab:ood}.

\subsection{BBBC021 MoA-Inception evaluation}
\label{app:moa}
\noindent\textbf{Classifier and evaluation protocol.}
We reuse the released CellFlux MoA classifier and evaluator \citep{zhang2025cellflux}. It extracts 2,048-dimensional features from uint8 cell images with the frozen Inception-v3 encoder used by the TorchMetrics FID implementation. A $2048\to512\to K$ multilayer perceptron, with ReLU and dropout 0.5, predicts one of the $K$ MoA labels in the CellFlux vocabulary. The classifier is trained on real perturbed training data using cross-entropy, Adam with learning rate 0.001, batch size 32, and 10 epochs. Preprocessing and augmentation follow the original CellFlux implementation. For each data split, the trained classifier is held fixed when evaluating real perturbed test data and the generated images from each method. In the standard-split comparison, generated RGB PNGs are decoded to $[0,1]$ and converted by $(255x).\mathrm{byte}()$ before the FID encoder. The released head uses 13 outputs including DMSO; the treated training and fixed test populations contain 12 and 9 labels. Its training uses independent horizontal and vertical flips with probability 0.3 and drops incomplete training batches.

During evaluation, the highest-scoring classifier output is compared with the biological label of the requested perturbation.

\noindent\textbf{Accuracy and F1.}
Accuracy is the fraction of correctly classified images. Macro-F1 averages classwise F1 equally, while weighted-F1 weights each class by its true-label sample count. The class set is the union of true and predicted labels, following the original evaluator; undefined classwise F1 terms are set to zero. All three metrics are reported as percentages, using the same definitions for the CellProfiler-based classifier below.

\subsection{BBBC021 MoA evaluation with cell profiling features}
\label{app:cp40}
\noindent\textbf{Images and morphology.}
We evaluate IMPA, CellFlux, Bridge-only, and PerturbBridge generated images. After excluding the eight held-out compounds in the CellFlux split, we use the first 5,120 treated test images in dataset order, with matched sample identifiers across real perturbed data and generated-image sets. Images remain native $96\times96$ uint8 RGB, without resizing or added contrast enhancement.

\noindent\textbf{CellProfiler extraction and preprocessing.}
CellProfiler 4.2.8 native modules provide 2,106 candidate measurements of shape, intensity, texture, granularity, radial distribution, and colocalization, independently of the Inception-v3 encoder used for biological response supervision. Pixels are divided by 255 without channelwise percentile normalization. Native primary-object, secondary-object, and cytoplasm segmentation define the measured regions. Each image contributes measurements from the primary object closest to its center (within 24 pixels) and the corresponding regions, without averaging or taking medians across cells. Feature filtering on real perturbed training data yields $d_{\rm CP}\leq2,106$ retained coordinates. Samples with failed segmentation or no qualifying object are retained, with missing measurements imputed by reference-feature medians; no central-circle fallback is used. The feature mask and imputation statistics are fixed from real perturbed training data and applied to all evaluation sets, giving $f\in\mathbb R^{d_{\rm CP}}$ per image.

\noindent\textbf{Normalization and classification using only real perturbed training data.}
For feature $j$, let $m_j$, $s_j$, and $d_j$ be the median, standard deviation, and scaled MAD on unflipped real perturbed training data:
\begin{align}
 d_j&=1.4826\,\operatorname{median}_i|f_{ij}-m_j|,
 & a_j&=\max(0.001s_j,10^{-6}),\\
 v_j&=\begin{cases}d_j,&d_j>a_j,\\\max(s_j,a_j),&\text{otherwise},\end{cases}
 &z_j&=\operatorname{clip}\!\left(\frac{f_j-m_j}{v_j},-12,12\right).
\end{align}
The normalization parameters fitted to real perturbed training data transform every evaluation set. A $d_{\rm CP}\to512\to13$ MLP with ReLU and dropout 0.5 is trained with cross-entropy, Adam (learning rate 0.001), batch size 32 with the final incomplete batch dropped, and 10 epochs. Horizontal and vertical flips are sampled independently with probability 0.3 from precomputed original/H/V/HV features. The final classifier is frozen, with no test-time flips or test-based epoch selection. Its 13 outputs retain the original vocabulary including DMSO; treated training data contain 12 labels and this test subset contains 9. Accuracy is per-image correctness. Macro- and weighted-F1 use the union of true and predicted labels with undefined terms set to zero. No test or generated images are used to fit the scaler or classifier.

\subsection{DINOv2 perturbation retrieval}
\label{app:retrieval}
Both evaluations use DINOv2 features. JUMP evaluates individual cells; AllenCell evaluates condition-level profiles. Their candidate sets and aggregation levels differ.

\noindent\textbf{JUMP.}
There are 9,077 single-cell test queries. Candidates are ranked by increasing Mahalanobis distance in DINOv2 feature space. The gallery contains real perturbed test cells, with the query's entire plate excluded. Cells with the same perturbation identity are positives. The JUMP retrieval results use PerturbBridge trained for 140 epochs and the Bridge-only baseline. Table~\ref{tab:retrieval} uses the CellFlux outputs; the real perturbed data reference uses queries from real perturbed data under the same exclusion rule.

\noindent\textbf{AllenCell.}
The evaluation contains 24 condition-profile queries. Each query retrieves among 2--5 compound candidates within the same target, using profiles from real perturbed training data as the gallery. The AllenCell retrieval results use PerturbBridge trained for 240 epochs, the Bridge-only baseline, and CellFlux baseline. Table~\ref{tab:retrieval} compares these methods with real perturbed data. This is condition-profile retrieval within a target, without cross-plate exclusion or an unrestricted search over all target--compound combinations. The DINOv2 representation is separate from the Inception-v3 features used by the biological response losses.

\noindent\textbf{Ranking metrics.}
Hit@$k$ is the percentage of queries with at least one positive among the first $k$ eligible candidates. MRR averages the reciprocal of the first-positive rank. For each query, average precision averages precision at the rank of every positive in the eligible gallery; mAP averages this quantity across queries. JUMP reports the median first-positive rank, whereas AllenCell reports the mean first-positive rank. These definitions use the dataset-specific galleries and positive sets described above.

\noindent\textbf{Metric reporting.}
Unless otherwise noted, all quantitative results from our experiments are reported as the mean over three random seeds.
FID and KID use Inception features; FDD and KDD apply Fr\'echet and kernel distribution comparisons in DINOv2-L/14 feature space. Table~\ref{tab:distribution} reports KID multiplied by 100 and unscaled KDD. Dashes indicate unavailable values. Classification and retrieval hit rates are percentages.
Table~\ref{tab:seed-std} reports the available standard deviations for PerturbBridge across three seeds. FID standard deviations range from 0.23 to 0.53 across datasets, while the reported BBBC021 MoA accuracy standard deviation is 0.350 percentage points. These values describe variation across runs, rather than uncertainty across individual cells. The displayed zero for AllenCell Hit@1 is at the reported precision and does not imply an absence of uncertainty. Standard deviations should be read in the units of each metric.

\begin{table}[!htbp]
\centering\small
\caption{Standard deviations across three random seeds for PerturbBridge on BBBC021, JUMP, and AllenCell; reported results are three-seed means. KID is multiplied by 100. MRR and mAP are multiplied by 100 on JUMP and remain unscaled on AllenCell, matching Table~\ref{tab:retrieval}. MoA accuracy, F1 scores, and Hit@1 standard deviations are expressed in percentage points (p.p.). Standard deviations are rounded to three decimal places; dashes indicate unavailable statistics.}
\label{tab:seed-std}
\setlength{\tabcolsep}{8pt}
\begin{tabular}{@{}lrrr@{}}
\toprule
Metric & BBBC021 & JUMP & AllenCell\\
\midrule
FID & 0.493 & 0.230 & 0.530\\
KID $\times100$ & 0.072 & 0.039 & 0.126\\
FDD & 6.855 & 3.410 & 7.940\\
KDD & 0.089 & 0.065 & 0.095\\
MoA Accuracy (p.p.) & 0.350 & --- & ---\\
MoA Macro-F1 (p.p.) & 0.887 & --- & ---\\
MoA Weighted-F1 (p.p.) & 0.450 & --- & ---\\
Retrieval Hit@1 (p.p.) & --- & 0.110 & 0.000\\
MRR & --- & 0.192 & 0.001\\
mAP & --- & 0.009 & 0.001\\
\bottomrule
\end{tabular}
\end{table}

\section{Pseudocode}
\label{app:pseudocode}
Algorithms~\ref{alg:train} and~\ref{alg:sample} summarize training and generation. The predictor returns the preconditioned predicted perturbed image. Bridge coefficients and sampling settings are defined in App.~\ref{app:sampler}; biological response losses and their weights are defined in App.~\ref{app:implementation}.

\texttt{drop\_condition} applies batchwise classifier-free dropout. The frozen encoder passes gradients to the predicted perturbed image. The weights of the two biological response losses specify the loss combination for each dataset. In sampling, \texttt{correct\_prediction} applies the BBBC cached-prediction correction and otherwise returns the current prediction; \texttt{noise\_strength} is zero for the final two transitions.

\par\medskip
\noindent\begin{minipage}{\linewidth}
\noindent\textbf{Training.}
Algorithm~\ref{alg:train} constructs a noisy bridge state and predicts the perturbed image in a single training pass. Bridge regression is always active; batches retaining the perturbation condition additionally receive moment and MoA losses through the frozen encoder. The time gate and training-progress schedule control the contributions of these losses.

\begin{algorithm}[H]
\caption{PerturbBridge: training objective.}
\label{alg:train}
\begin{lstlisting}[style=paperpython]
# net: perturbed-image predictor; phi: frozen encoder
# x0, y, c: grouped target images, controls, conditions
# stats, prototypes: fixed references from real perturbed training data

def training_loss(net, x0, y, c, progress):
    t = sample_time(x0)
    z = randn_like(x0)
    alpha, beta, gamma = bridge_coefficients(t)
    xt = alpha * x0 + beta * y + gamma * z
    cond = drop_condition(c)
    xhat = net(xt, t, cond)
    loss = mean(weight(t) * (xhat - x0)**2)

    # Biological response supervision of the predicted image
    lm, la = bio_weights(progress)
    if cond is not None:
        f = phi(preprocess(xhat))
        g = time_gate(t)
        loss += lm * moment_loss(f, c, stats, g)
        loss += la * moa_loss(f, c, prototypes, g)
    return loss
\end{lstlisting}
\end{algorithm}
\end{minipage}

\par\medskip
\noindent\begin{minipage}{\linewidth}
\noindent\textbf{Generation.}
Algorithm~\ref{alg:sample} initializes the state from a control image and applies classifier-free guidance to successive perturbed-image predictions. The known control and analytic bridge coefficients enter each update explicitly. Stochastic transitions are followed by deterministic finishing steps, without evaluating biological losses during generation.

\begin{algorithm}[H]
\caption{PerturbBridge: sampling.}
\label{alg:sample}
\begin{lstlisting}[style=paperpython]
# net: trained perturbed-image predictor; y: control image
# times: decreasing grid; s: CFG coefficient

def generate(net, y, c, times, s):
    x, previous = y, None
    for t, r in adjacent_pairs(times):
        q = (1 + s) * net(x, t, c) - s * net(x, t, None)
        xhat = correct_prediction(q, previous, t, r)
        a, b, g = bridge_coefficients(t)
        zhat = (x - a * xhat - b * y) / g
        eps = noise_strength(t)

        if eps > 0:
            da, db, dg = bridge_derivatives(t)
            drift = da*xhat + db*y + (dg + eps/g)*zhat
            h = t - r
            x = x - h*drift + sqrt(2*eps*h)*randn_like(x)
        else:
            ar, br, gr = bridge_coefficients(r)
            x = ar*xhat + br*y + gr*zhat
        previous = (q, t)
    return x
\end{lstlisting}
\end{algorithm}
\end{minipage}

\section{Biological response losses: implementation and schedules}
\label{app:implementation}

BBBC021, JUMP, and AllenCell share the same two biological response objectives: condition-moment matching and InfoNCE-style contrastive supervision. Positive and negative sets are constructed according to each dataset's condition definitions and available biological annotations. On BBBC021, prototypes sharing the requested MoA are positives, while prototypes from other MoAs are negatives. The formulation below illustrates these shared objectives using BBBC021; both supervise the training-time predicted perturbed image $\hat x_0=D_\theta(X_t,t,c)$.

\subsection{Differentiable features and caches of real perturbed training data}
The frozen Torch-Fidelity Inception-v3 encoder $\phi$ returns 2,048 features in evaluation mode. Its preprocessing takes the first three channels of $\hat x_0$ and uses straight-through clipping and centered one-bin dither. With $\operatorname{sg}$ denoting stop-gradient,

\begin{align}
\operatorname{STclip}_{[a,b]}(v)&=v+\operatorname{sg}(\operatorname{clip}_{[a,b]}(v)-v),\\
u_i&=\operatorname{STclip}_{[0,1]}\left[
\frac{\operatorname{STclip}_{[-1,1]}(\hat x_{0,i})+1}{2}
+\frac{\delta_i}{255}\right],\quad\delta_i\sim U[-\tfrac12,\tfrac12],\\
f_i&=\phi(255u_i).
\label{eq:bio-interface}
\end{align}

The encoder runs in FP32. Its parameters remain fixed, while gradients through $f_i$ reach $\hat x_{0,i}$ and the generator. Offline MoA evaluation uses uint8 images as described in App.~\ref{app:moa}; it does not add the dither used for biological response supervision.

The cache uses all 63,781 images in the real perturbed training data, without flips, with ordinary clipping and seeded dither (seed 20260810). For each compound condition $c$, it stores $\mu_c=\E[f\mid c]$ and diagonal population variance $v_c=\E[f^{\odot2}\mid c]-\mu_c^{\odot2}$. The map $A(c)$ supplies its MoA label. Conditions remain separate when they share a MoA.

\subsection{Time gate and condition-moment matching}
The two biological response losses share a smooth per-image gate:

\begin{equation}
g_i=\operatorname{sigmoid}\left(\frac{a-|c_{\rm out}(t_i)|}{b}\right),
\quad a=\frac{1}{\sqrt{13}},\quad b=\frac{a}{10}.
\label{eq:bio-gate}
\end{equation}

Smaller $|c_{\rm out}|$ gives a larger gate. For a local-batch condition group $I_c$, the weighted moments and valid groups are

\begin{align}
G_c&=\sum_{i\in I_c}g_i,\quad Z_c=\max(G_c,10^{-6}),\\
\hat\mu_c&=\frac{1}{Z_c}\sum_{i\in I_c}g_i f_i,\qquad
\hat v_c=\max\left\{\frac{1}{Z_c}\sum_{i\in I_c}g_if_i^{\odot2}-\hat\mu_c^{\odot2},0\right\},\\
\mathcal V&=\{c:|I_c|\geq4,\ G_c>10^{-6}\}.
\label{eq:bio-moments-exact}
\end{align}

We insert these moments into Equation~\eqref{eq:mom}, with $\mathcal C_B=\mathcal V$, $d_f=2048$, and $\lambda_{\mathrm{var}}=0.5$, averaging equally over valid conditions and feature dimensions. Each group contains eight cells; groups with fewer than four cells are excluded. The moment loss is zero when no group is valid.

\subsection{Multi-positive MoA prototype contrast}
For BBBC021, we denote the general contrastive loss $\mathcal L_{\rm con}$ by $\mathcal L_{\rm MoA}$ and its weight $\lambda_{\rm con}$ by $\lambda_{\rm MoA}$, reflecting the use of MoA labels to define positive prototypes.
The cache supplies one unit-normalized prototype $p_j$ per training condition, and $z_i=f_i/\|f_i\|_2$ is the normalized predicted-image feature. The 26 BBBC021 prototypes define
\begin{equation}
q_{ij}=\frac{\exp(z_i^\top p_j/\tau)}{\sum_{k=1}^{26}\exp(z_i^\top p_k/\tau)},
\qquad \tau=0.1.
\label{eq:softmax}
\end{equation}
All prototypes sharing the requested MoA, including the prototype of the requested condition, are positives in Equation~\eqref{eq:moaloss}. The implementation uses equal temperature and base temperature, so their ratio is one. The batch reduction is
\begin{equation}
\mathcal L_{\rm MoA}=\frac{\sum_i g_i\ell_i^{\rm MoA}}{\max(\sum_i g_i,10^{-12})}.
\label{eq:bio-supcon-exact}
\end{equation}
This retains the within-positive-set term in Proposition~\ref{prop:contrast}. Moment matching preserves compound-specific feature statistics, while MoA contrast supervises shared biological identity. Both pass gradients through the frozen encoder.

\subsection{Biological response loss weight schedule}
For BBBC021, biological response supervision is inactive for the first 40\% of training. The moment and MoA weights then increase linearly from zero to 0.1 and 0.01, respectively, in the joint objective of Equation~\eqref{eq:total}. The gate $g_i$ weights samples within each loss, while this schedule controls their overall contribution. Batches with dropped perturbation conditions use bridge regression alone.
\section{Qualitative examples of generated cells}
\label{app:qualitative}
These examples complement the biological evaluations in Section~\ref{sec:experiments}. They compare selected real reference images with multiple PerturbBridge samples under the indicated conditions. The comparisons illustrate morphology and variation among displayed cells; they are not paired reconstructions or quantitative tests of population diversity.

\noindent\textbf{Overview across seven structures.}
Fig.~\ref{fig:allen-structure-overview} organizes the comparison by compound and imaged structure, placing PerturbBridge samples and real references in parallel grids. The vehicle-control row provides context for the treatment rows, and highlighted entries direct attention to selected differences in subcellular organization. Blank entries denote combinations absent from the dataset. This overview extends the selected compound--structure comparisons in Fig.~\ref{fig:generated-real}.

\begin{figure}[H]
\centering
\includegraphics[width=\linewidth]{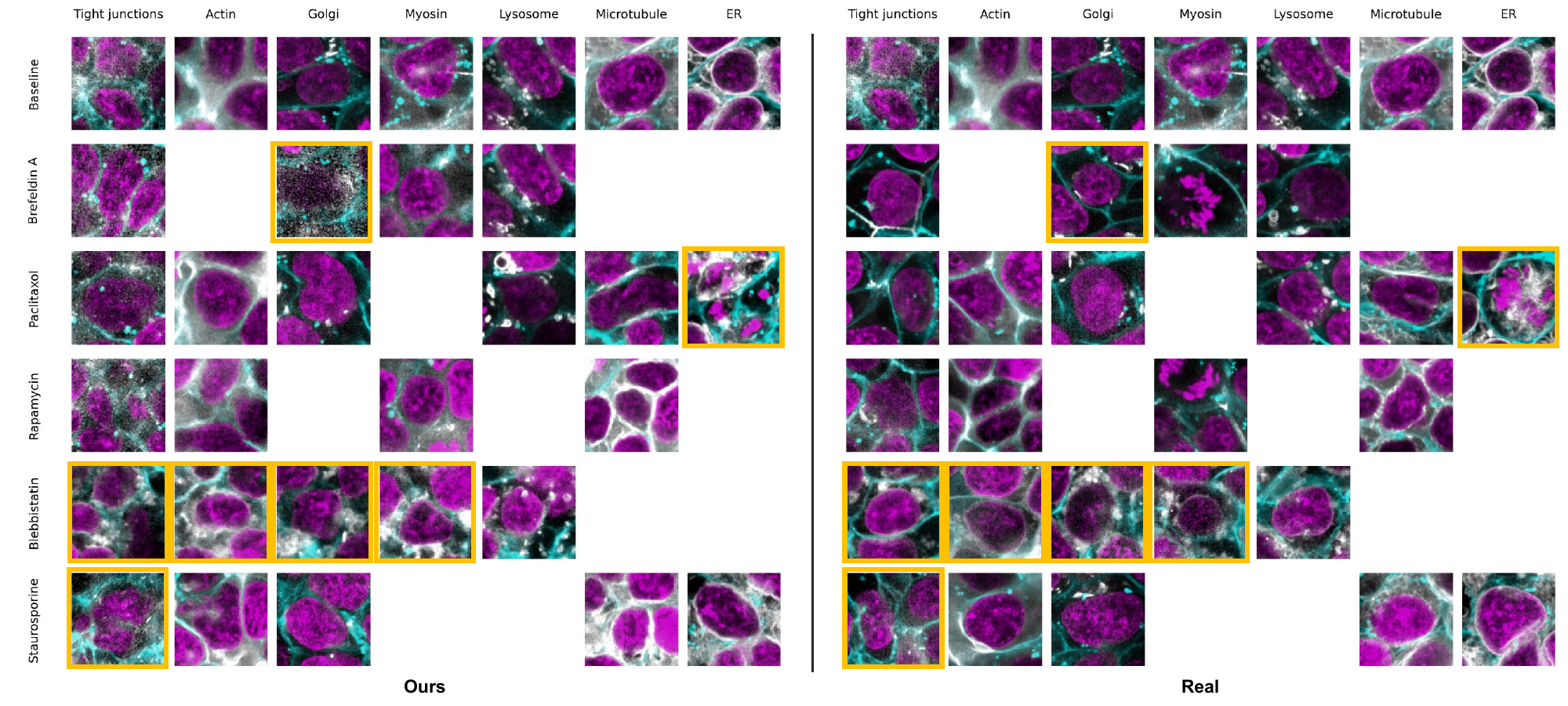}
\caption{Drug-induced changes across seven cellular structures on AllenCell. PerturbBridge samples (left) and real references (right) are arranged by compound and structure. Vehicle denotes the control; SNB denotes (S)-nitro-blebbistatin. Yellow frames highlight selected comparisons; blank entries denote combinations absent from the dataset.}
\label{fig:allen-structure-overview}
\end{figure}

\subsection{BBBC021: examples across compound conditions}
\par\medskip
\noindent\begin{minipage}{\linewidth}
\noindent\textbf{Cytochalasin B and AZ138.}
Fig.~\ref{fig:bbbc-qualitative-1} contrasts two conditions with different MoA annotations. The cytochalasin B examples show irregular cell outlines and variation in peripheral staining, while the AZ138 examples include rounded or elongated cells with bright nuclear regions. Generated samples vary in shape and staining within each condition, allowing visual comparison with the real perturbed references. MoA preservation is assessed quantitatively in Tables~\ref{tab:cp40} and~\ref{tab:moa}.

\begin{figure}[H]
\centering
\includegraphics[page=1,width=\linewidth,height=0.72\textheight,keepaspectratio]{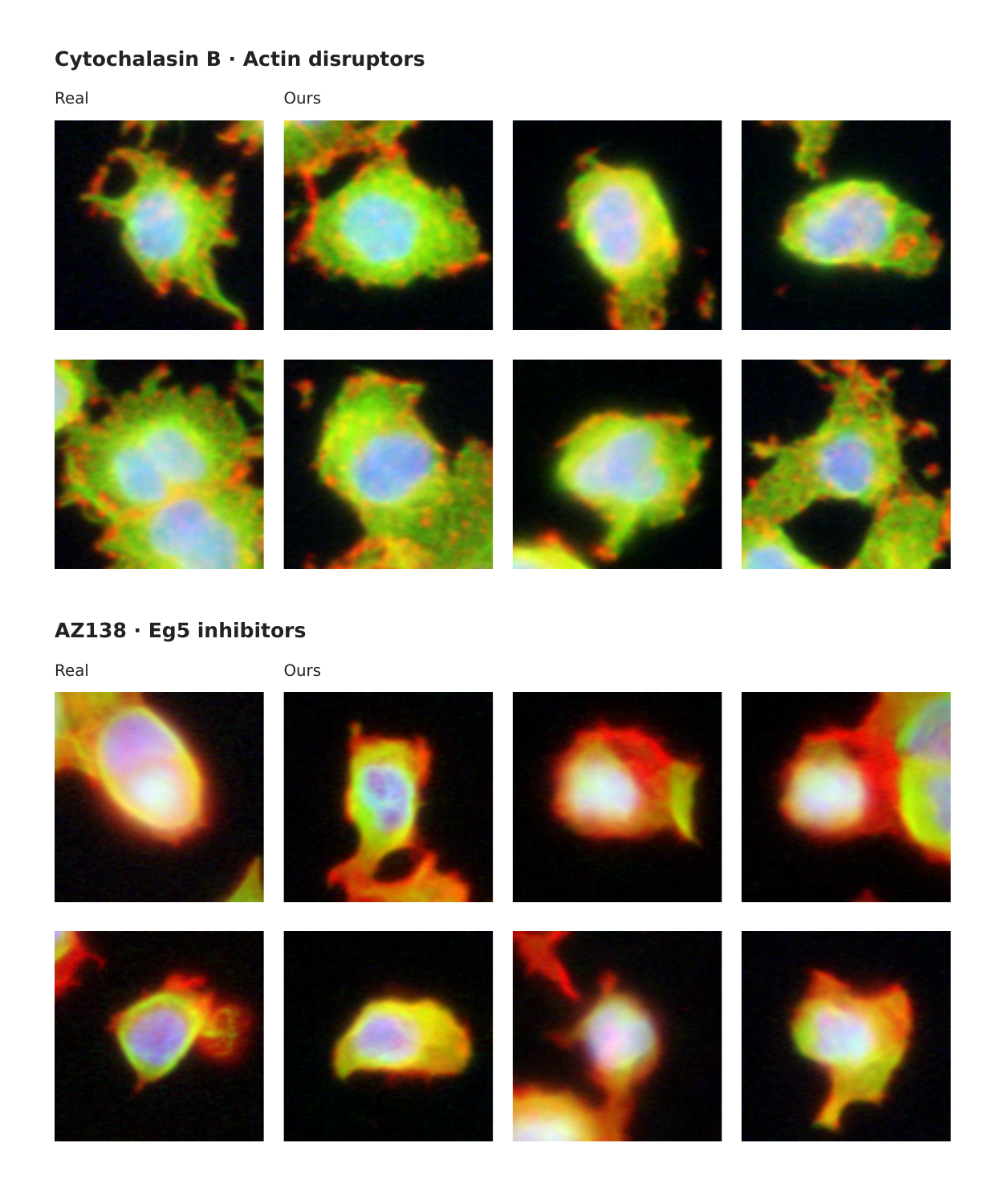}
\caption{BBBC021 examples for cytochalasin B and AZ138. In each compound block, the left column contains real perturbed reference images and the remaining columns contain PerturbBridge samples.}
\label{fig:bbbc-qualitative-1}
\end{figure}
\end{minipage}

\par\medskip
\noindent\begin{minipage}{\linewidth}
\noindent\textbf{AZ258 and PP-2.}
Fig.~\ref{fig:bbbc-qualitative-2} extends the comparison to two additional compound conditions. The displayed cells differ in nuclear shape, cell extent, and the spatial distribution of channel intensities. Multiple generated images show how these properties vary under the same requested condition. Individual generated cells need not reproduce the geometry of a particular reference cell, because training and evaluation concern conditional cell populations.

\begin{figure}[H]
\centering
\includegraphics[page=2,width=\linewidth,height=0.72\textheight,keepaspectratio]{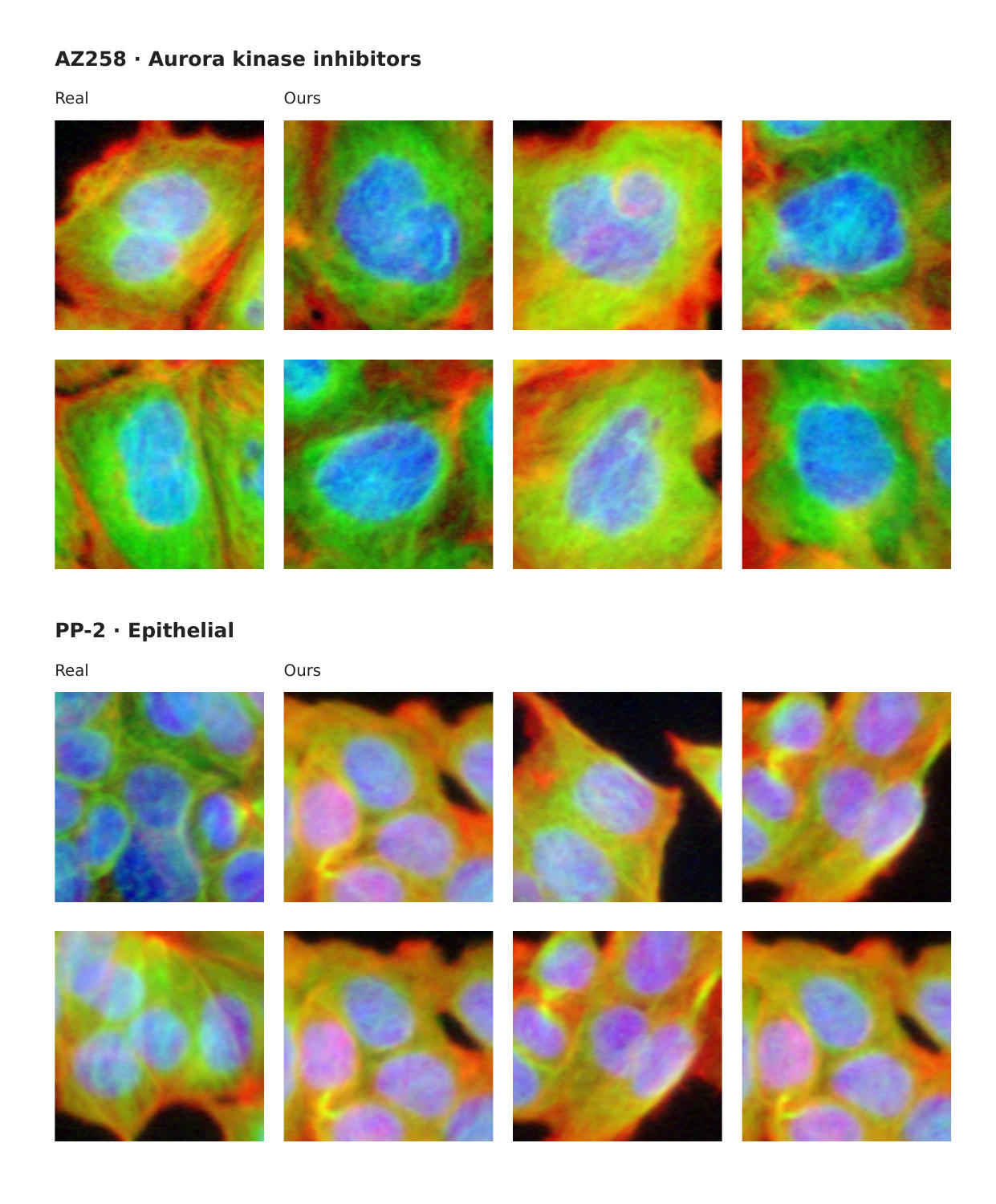}
\caption{Additional BBBC021 examples for AZ258 and PP-2. Real perturbed references appear in the left column and PerturbBridge samples in the remaining columns.}
\label{fig:bbbc-qualitative-2}
\end{figure}
\end{minipage}

\subsection{AllenCell: structure-specific comparisons}

\par\medskip
\noindent\begin{minipage}{\linewidth}
\noindent\textbf{Brefeldin A and paclitaxel.}
Fig.~\ref{fig:allen-qualitative-1} provides larger views of Golgi and microtubule imaging conditions. Vehicle and treated references allow changes in the location and organization of the fluorescent signal to be inspected alongside multiple generated samples. The samples exhibit variation in nuclear geometry and surrounding staining, rather than repeating a single cell appearance. These selected views complement the condition-profile retrieval results in Table~\ref{tab:retrieval}.

\begin{figure}[H]
\centering
\includegraphics[page=1,width=\linewidth,height=0.72\textheight,keepaspectratio]{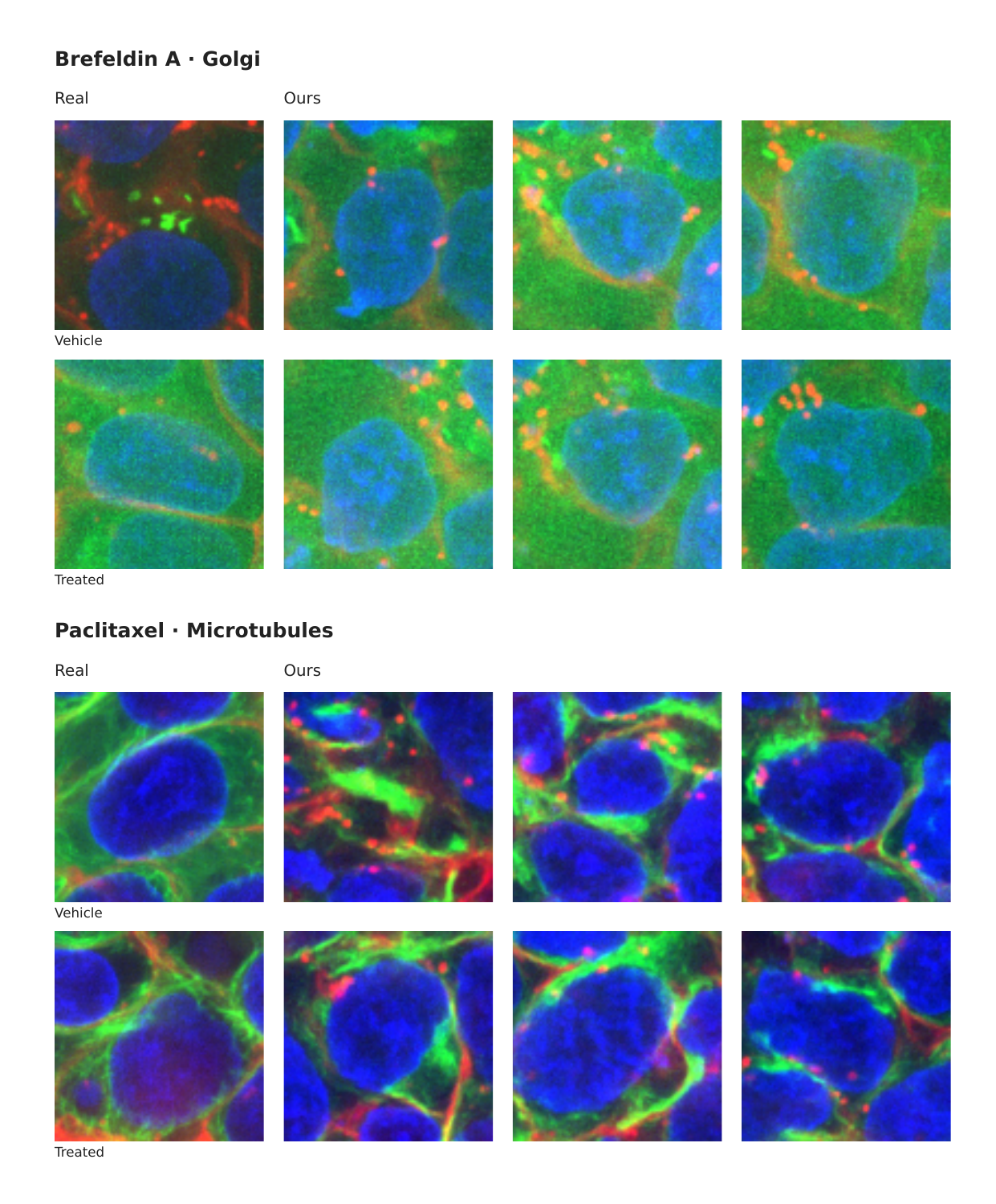}
\caption{AllenCell examples for brefeldin A (Golgi) and paclitaxel (microtubules). The left column shows real vehicle and treated references; the remaining columns show PerturbBridge samples.}
\label{fig:allen-qualitative-1}
\end{figure}
\end{minipage}

\par\medskip
\noindent\begin{minipage}{\linewidth}
\noindent\textbf{Two imaging targets under the same compound.}
Fig.~\ref{fig:allen-qualitative-2} shows SNB examples for Golgi and tight-junction imaging. The two blocks emphasize different spatial features of the cellular response, so each should be compared with its corresponding vehicle and treated references. Together with the preceding examples, this comparison illustrates the distinct spatial appearances represented under different imaging targets.

\begin{figure}[H]
\centering
\includegraphics[page=2,width=\linewidth,height=0.72\textheight,keepaspectratio]{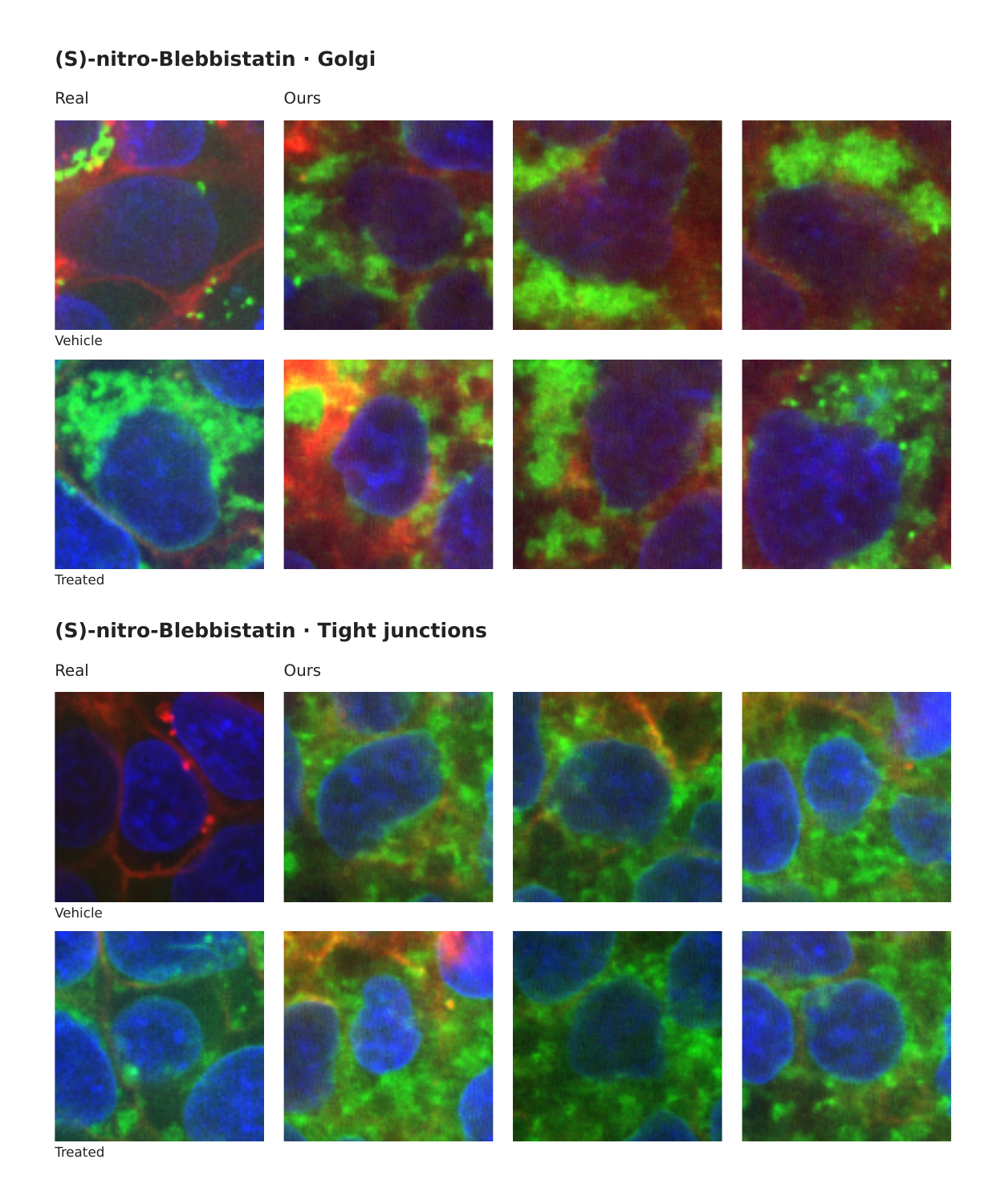}
\caption{AllenCell examples for (S)-nitro-blebbistatin (SNB) under Golgi and tight-junction imaging. Real vehicle and treated references appear in the left column; PerturbBridge samples occupy the remaining columns.}
\label{fig:allen-qualitative-2}
\end{figure}
\end{minipage}

\end{document}